\documentclass[journal,12pt,onecolumn,draftclsnofoot,]{IEEEtran}
\IEEEoverridecommandlockouts
\usepackage{cite}
\usepackage{amsmath,amssymb,amsfonts}
\usepackage{graphicx}
\usepackage{subcaption}
\usepackage{amsthm}
\usepackage{mathtools}
\usepackage{thmtools}
\usepackage{thm-restate}
\usepackage[ruled,norelsize,linesnumbered]{algorithm2e}
\usepackage{xcolor}
\def\BibTeX{{\rm B\kern-.05em{\sc i\kern-.025em b}\kern-.08em
    T\kern-.1667em\lower.7ex\hbox{E}\kern-.125emX}}
\usepackage{url}
\usepackage{ifthen}
\usepackage{cite}
\usepackage{bm}
\usepackage{bbm}
\usepackage{tikz}
\usepackage{hyperref}
\usepackage{cleveref}
\usepackage{centernot}
\usepackage{romannum}
\usepackage{multirow}

\makeatletter
\newcommand{\removelatexerror}{\let\@latex@error\@gobble}
\makeatother

\newtheorem{theorem}{Theorem}

\newtheorem{corr}{Corollary}

\newtheorem{remark}{Remark}
\newtheorem{lemma}{Lemma}
\newtheorem{definition}{Definition}

\newcommand{\norm}[1]{\left\lVert#1\right\rVert}
\newcommand{\defeq}{\vcentcolon=}

\newcommand{\hatvect}[1]{\bm{\hat{#1}}}
\newcommand{\vecthat}[1]{\bm{\hat{#1}}}
\newcommand{\vectcheck}[1]{\bm{\check{#1}}}
\newcommand{\vect}[1]{\bm{#1}}
\DeclareMathOperator*{\argmin}{arg\,min}

\newcommand{\expectation}{\mathbb{E}}
\newcommand{\independent}{\perp \!\!\! \perp}

\newcommand{\paramstep}[2]{\ensuremath{{#1}}^{(#2)}}
\newcommand{\innerprod}[2]{\ensuremath{\langle #1,#2 \rangle}}
\newcommand{\floor}[1]{\ensuremath{\lfloor #1 \rfloor}}

\newcommand{\Lloss}{\ensuremath{\mathcal{L}}}
\newcommand{\maxone}[1]{\ensuremath{(#1)^{++}}}

\usepackage{color, soul}

\newlength\mylen

\begin{document}

\title{A Machine Learning Framework for Distributed Functional Compression over Wireless Channels in IoT}

\author{
	\IEEEauthorblockN{Yashas Malur Saidutta, Afshin Abdi, and Faramarz Fekri}
		\vspace{-10mm}
		\thanks{Yashas Malur Saidutta was with the Department of Electrical and Computer Engineering at Geogia Institute of Technology, Atlanta, GA 30318 USA. He is now with the Samsung Research America, Mountain View, CA-94043, USA. All work was performed while with Georgia Institute of Technology. (e-mail: yashas.saidutta@gatech.edu). Afshin Abdi was with the Department of Electrical and Computer Engineering at Geogia Institute of Technology, Atlanta, GA-30318, USA. He is now with the Qualcomm Technologies, Inc., San Diego, CA-92121, USA (e-mail:abdi@gatech.edu). Faramarz Fekri is with the Department of Electrical and Computer Engineering, Geogia Institute of Technology, Atlanta, GA-30318, USA (e-mail: faramarz.fekri@gatech.edu).}
}

\maketitle

\begin{abstract}
	IoT devices generating enormous data and state-of-the-art machine learning techniques together will revolutionize cyber-physical systems. In many diverse fields, from autonomous driving to augmented reality, distributed IoT devices compute specific target functions without simple forms like obstacle detection, object recognition, etc. Traditional cloud-based methods that focus on transferring data to a central location either for training or inference place enormous strain on network resources. To address this, we develop, to the best of our knowledge, the first machine learning framework for distributed functional compression over both the Gaussian Multiple Access Channel (GMAC) and orthogonal AWGN channels. Due to the Kolmogorov-Arnold representation theorem, our machine learning framework can, by design, compute any arbitrary function for the desired functional compression task in IoT. Importantly the raw sensory data are never transferred to a central node for training or inference, thus reducing communication. For these algorithms, we provide theoretical convergence guarantees and upper bounds on communication. Our simulations show that the learned encoders and decoders for functional compression perform significantly better than traditional approaches, are robust to channel condition changes and sensor outages. Compared to the cloud-based scenario, our algorithms reduce channel use by two orders of magnitude.
	
\end{abstract}

\begin{IEEEkeywords}
Internet of Things, Distributed Functional Compression, Noisy Wireless Channels, Deep Learning.
\end{IEEEkeywords}

\section{Introduction}
Internet of Things (IoT), set to revolutionize cyber-physical systems, is underpinned by the recent breakthroughs in computing, communication, and machine learning. A majority of the 75 billion IoT devices will be connected over wireless networks and will be collecting close to two exabytes of data per day. Combining such staggering level of data with machine learning can lead to unprecedented applications. In many diverse areas like autonomous driving \cite{baek2020vehicle}, chemical/nuclear power plant monitoring \cite{li2017using}, environment monitoring \cite{ullo2020advances}, and augmented reality \cite{Choudhary20cvpr}, distributed IoT devices collectively compute specific target functions without simple known forms like failure prediction, obstacle detection, etc. Traditional cloud-based methods focus on transferring the edge data to a central location for model training, which places tremendous stress on the network resources. Alternatively, we develop a machine learning framework where we leverage distributed training data to learn models that compute a specific function in such a way that the raw data never leaves the IoT device where it is collected.

We focus on the fundamental question, ``How to train a neural network model for a specific function $\mathcal{F}\left(\vect{x}_1,\dots,\vect{x}_N\right)$ without explicitly communicating the massive training data $(\vect{x}_1^{(b)},\dots,\vect{x}_N^{(b)})_{b=1}^{B}$ that is split among the $N$ nodes in a distributed wireless network?". In this setup shown in \Cref{fig:jscc_fc_awgn_all}, we will leverage the distributed training data to learn a global collaborative model that will help to compute a specific target function at a fusion center. Unlike classical ML applications, the model itself is distributed among the various nodes. Specifically, we use an edge router and sensor/edge nodes setup. The sensor nodes observe the data in a distributed manner and perform processing without coordination with other sensor nodes. Following which, the edge router serves as a fusion center for this processed data and approximates the value of some target function. The goal is to send relevant information to the edge router in a communication efficient manner.

\subsection{Related Work}

\noindent \textbf{Distributed Functional Compression:} There are many works in distributed functional compression. Works over orthogonal channels assume that the source observations are independent given the function value (like the CEO problem) \cite{berger1996ceo, viswanathan1997quadratic, prabhakaran2004rate, oohama1998rate, he2015lower, uugur2020vector}, employ asymptotic methods \cite{doshi2007distributed, doshi2010functional, feizi2014network}, or consider only linear functions \cite{krithivasan2009lattices, lalitha2013linear, wagner2010distributed}. Works over GMAC focus on simple target functions like linear functions, geometric mean, etc. \cite{nazer2007computation, kortke2014analog, goldenbaum2013harnessing, goldenbaum2013robust, goldenbaum2012analog}. The works leveraging distributed data for learning \cite{hanna2020distributed, jankowski2020joint, aguerri2017distributed, zaidi2020distributed} do not focus on communication efficiency during training.

\noindent \textbf{Distributed and Federated Learning:} In the general context of machine learning, \cite{xu2020acceleration, gupta2020training} studied model parallelism which, is more useful in a cluster environment rather than wireless channels. In split learning, works focus on reducing communication by using GMAC instead of orthogonal channels \cite{krouka2021communication} or architectural changes \cite{koda2020communication}.

\subsection{Contributions}
\begin{enumerate}
	%\item We propose three deep learning based methods to perform distributed  functional compression over Gaussian MAC and AWGN channels.
	\item To the best of our knowledge, we develop the first machine learning framework for distributed functional compression over GMAC and orthogonal AWGN channels that leverage distributed training data collected using smart IoT devices.
	\item We develop a three-stage algorithm for distributed training where the raw sensory data never leaves the IoT devices, either for training or inference. 
	\item We exploit the channel structure and the classification nature of the target function to reduce communication from sensor node to edge router by completely removing end-to-end training. We show that this algorithm converges to a stationary point of the non-convex loss function and the number of communication rounds $T< \mathcal{O}\left(\frac{N}{\delta^2}\right)$, where $\delta$ is the minimum norm of the gradient encountered, and $N$ is the number of nodes.
\end{enumerate}

\noindent \emph{Notations:} Upper case letters denote random variables, and bold upper case denotes random vectors.

\section{Distributed Wireless-ML Engine for Functional Compression}
\label{sec:prob_def}
\begin{figure}[h]
	\vspace{-4mm}
	\centering
	\includegraphics[width=0.5\linewidth]{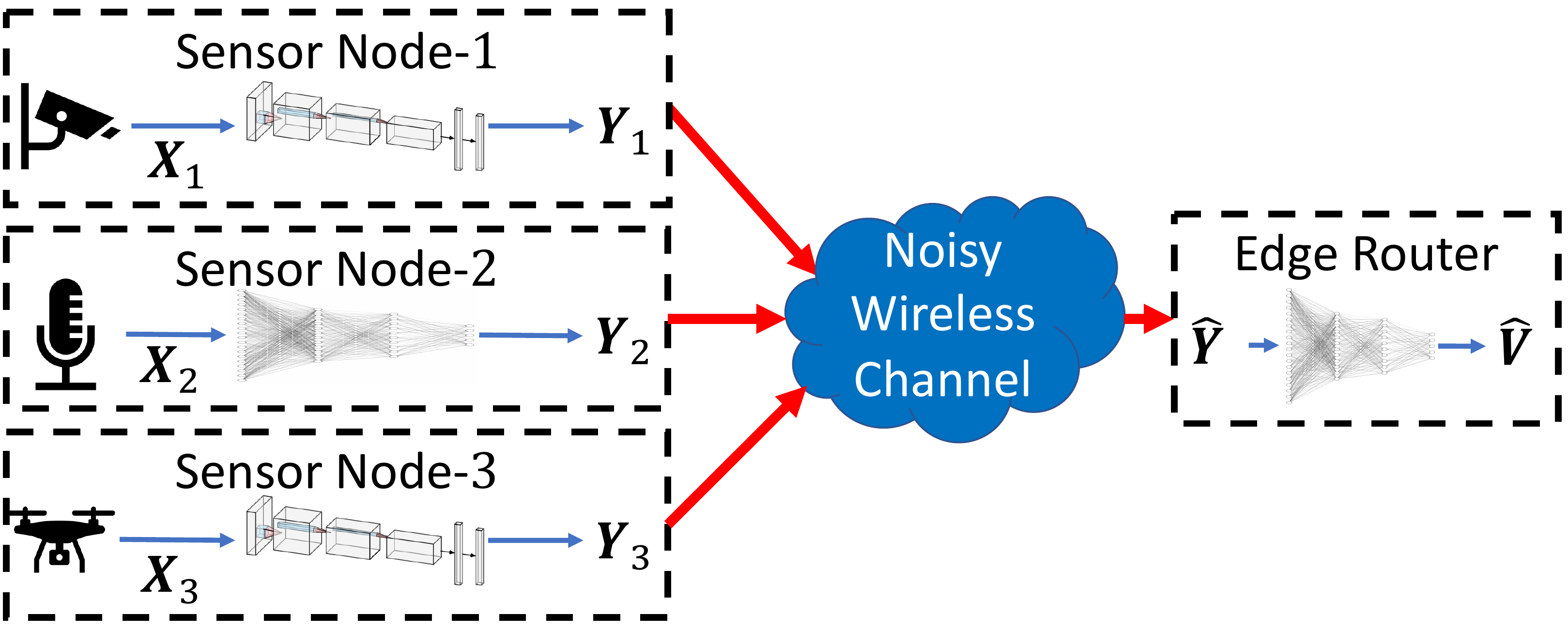}
	\caption{Distributed wireless-ML engine for functional compression with $N=3$ sensor (edge) nodes.}
	\label{fig:jscc_fc_awgn_all}
	\vspace{-8mm}
\end{figure}

Consider a setup with $N$ spatially separated sensor nodes and an edge router, as shown in \Cref{fig:jscc_fc_awgn_all}. We use $n \in \{1,\dots,N\}$ to index the sensor nodes. Then, $\vect{X}_n \in \mathcal{X}_n \subseteq \mathbb{R}^{M_n}$ is the random variable that represents the information source observed by node-$n$. The edge router attempts to recover some target function, $\vect{v}=\mathcal{F}(\vect{x}_1,\dots,\vect{x}_N)$. In fact, except for the discussion in \cref{subsec:distributed_training_classification}, our methods can be used for the more general problem of approximating the value of a random variable $\vect{V}$ where $I(\vect{X}^N;\vect{V})\geq 0$. %Functional compression is a special case of this formulation.

Each sensor node-$n$ has an encoding function $g_e^{(n)}: \mathcal{X}_n \rightarrow \mathcal{Y}_n$,  where $\mathcal{Y}_n \subseteq \mathbb{R}^{K_n}$. $K_n$ represents the number of channel uses for node-$n$. Further, each sensor node-$n$ is subject to power constraint $\frac{1}{K_n}\expectation_{\vect{X}_n} \left[\norm{\vect{Y}_n}_2^2\right] \leq P_T^{(n)}$, where $\expectation$ represents expectation. The noisy channel $h$ maps $\{\mathcal{Y}_1,\dots,\mathcal{Y}_N,\mathcal{Z}\} \rightarrow \mathcal{\hat{Y}}$, where $\vect{Z} \in \mathcal{Z}$ is an independent random vector that represents the randomness introduced by the channel and $\mathcal{\vecthat{Y}} \subseteq \mathbb{R}^K$. The random variable received at the edge router is $\vecthat{Y}$. The edge router employs a decoding function $g_d:\mathcal{\vecthat{Y}} \rightarrow \mathcal{\vecthat{V}}$ to give an estimate of $\mathcal{F}$.
Neural networks parametrize the encoding function at node-$n$ and the decoding function at the router with parameters $\vect{\Theta}_n$ and $\vect{\Theta}_{N+1}$ respectively.

Let us denote the set of encoding functions by $g_e \defeq \{g_e^{(1)},\dots,g_e^{(N)}\}$, $p(v,\vect{x}^N)$ be the joint distribution of $\left(V,\vect{X}^N\right)$, and $p(\vect{z})$ the distribtuion of $\vect{Z}$. The overall objective of this setup can be formalized as a constrained optimization problem of the form
\begin{equation}
		\argmin_{g_e,g_d} \expectation_{\vect{V},\vect{X}^N,\vect{Z}} \left[\mathcal{D}_{\vect{V}} \left( \vect{v}, \vecthat{v} \right) \right]
		\, \textrm{s.t.} \, \expectation_{\vect{X}_n} \left[\norm{\vect{Y}_n}_2^2\right] \leq K_n P_T^{(n)}. %\; \forall n \in \{1,\dots,N\}.
	\label{eqn:overall_opti}
\end{equation}
Here $\mathcal{D}_{\vect{V}}$ represents some distortion measure between $\vect{v}$ and $\vecthat{v}$, and $\vecthat{v} = g_d(h(g_e^{(1)}(\vect{x}_1),\dots,g_e^{(N)}(\vect{x}_N),\vect{z}))$. In this formulation, the power constraint implicitly enforces the rate constraint.  In this paper, the joint distribution of $V,\vect{X}^N$ is unknown, and instead, we use a set of i.i.d. samples $(v^{(b)},\vect{x}_1^{(b)},\dots,\vect{x}_N^{(b)})_{b=1}^{B}$, where $B$ is the number of samples.

We assume that the channel mapping is $h(\vect{Y}_1,\dots,\vect{Y}_N,\vect{Z}) = h_1 (\vect{Y}_1,\dots,\vect{Y}_N)+\vect{Z}$. 
\subsubsection{Orthogonal AWGN channel} The received $\vecthat{Y}$ is
\begin{equation}
	\vecthat{Y} = 
	\begin{bmatrix}
		\vecthat{Y}_1 \\
		\vdots \\
		\vecthat{Y}_N \\
	\end{bmatrix}
	=
	\begin{bmatrix}
		\vect{Y}_1 \\
		\vdots \\
		\vect{Y}_N \\
	\end{bmatrix}
	+
	\vect{Z}
	\label{eqn:awgn_def}
\end{equation}
Here, $\vect{Z} \thicksim \mathcal{N}\left(\vect{0}, \sigma^2_z \vect{I}_{K}\right)$ denotes the AWGN noise component and $\vect{I}_{K}$ is the identity matrix of dimension $K$, and $K=\sum_{n=1}^{N}K_n$. 

\subsubsection{Gaussian MAC} The received $\vecthat{Y}$ is
\begin{equation}
	\vecthat{Y} = \sum_{n=1}^{N} \vect{Y}_n + \vect{Z}.
\end{equation}
Here $\vect{Z} \thicksim \mathcal{N}\left(\vect{0}, \sigma^2_z \vect{I}_{K}\right)$ and $K=K_1=\dots=K_N$.

Interestingly for both channel models, when there is no channel noise, the setup in \Cref{fig:jscc_fc_awgn_all} can realize any \emph{arbitrary multivariate continuous function}. This follows from Hilbert's thirteenth problem and the Kolmogorov-Arnold representation theorem \cite{kolmogorov1961representation, arnold2009functions}, which showed that any function $\mathcal{F}(\vect{x}_1,\dots,\vect{x}_n)$ has a nomographic representation $\psi(\sum_{n=1}\phi_n (\vect{x}_n))$. This applies trivially to the case of noiseless GMAC. In the noiseless orthogonal channel setup, we can see this by decomposing the output as $g_d(\vect{y}^N) = {g}_d' (\sum_{n=1} W_d^{(n)} g_e^{(n)}(\vect{x}_n))$ where $[W_d^{(1)},\dots,W_d^{(N)}]^T$ is the decoder network's first layer's weight and $g_d'$ is the rest of the network. However, both channels need $\sum_n M_n$ transmissions \cite{ostrand1965dimension}. In other words, our ML framework, by design, does not lose any optimality in terms of realizing the function $\mathcal{F}(\cdot)$.

\section{A tale of three loss functions}
\label{sec:tale_of_three_loss_functions}
We consider three loss functions to learn the encoding and decoding functions, as described in the following \footnote{For a more detailed treatment of this section, please refer to our conference paper \cite{saidutta2021analog}}.

\subsection{Method 1: Autoencoder based learning}
If we can ensure that the power constraint is satisfied, then the optimization problem simply reduces to minimizing the distortion. To address the former requirement, we can normalize each $\vect{y}_n \; \forall \vect{x}_n \in \mathcal{X}_n$ i.e., $\norm{g_e^{(n)}(\vect{x}_n)}_2^2=K_n P_T^{(n)}$. Thus the minimization objective can be written as
\begin{equation}
	\mathcal{L}_{\textrm{A}} = \mathbb{E}_{\vect{V},\vect{X}^N,\vect{Z}} \left[ \mathcal{D}_{V}\left(\vect{v}, g_d\left( h\left( \{ g_e^{(n)}(\vect{x}_n) \}_{n=1}^{N}, \vect{z} \right) \right) \right) \right].
	\label{eqn:au_loss_function}
\end{equation}

\subsection{Method 2: Unconstrained Optimization}
We can also convert the constrained optimization problem in \eqref{eqn:overall_opti} to an unconstrained optimization problem using Lagrange multipliers. This gives us a minimization objective of the form
\begin{equation}
	\mathcal{L}_{\textrm{L}} = \mathbb{E}_{\vect{V},\vect{X}^N,\vect{Z}} \left[ \mathcal{D}_{V}\left(\vect{v}, g_d\left( h\left( \{ g_e^{(n)}(\vect{x}_n) \}_{n=1}^{N}, \vect{z} \right) \right) \right) \right. + \left. \sum_{n=1}^{N} \lambda_n \norm{g_e^{(n)}(\vect{x}_n)}_2^2 \right].
	\label{eqn:sl_loss_function}
\end{equation}
Here $\lambda_n \; \forall n \in \{1,\dots,N\}$ are the Lagrange multipliers.

\subsection{Method 3: Variational Information Botteleneck}
\label{subsec:vib}
In \cite{tishby2000information}, Tishby et. al. proposed the Information Bottleneck (IB) theory as a generalization to the Rate-Distortion theory of Shannon \cite{cover1999elements}. They let another RV of interest $\vect{V}$ dictate what features are conserved in the compressed representation $\vecthat{Y}$ of the source $\vect{X}^N$. Finding the compressed representation $p(\vecthat{y}|\vect{x}^N)$ is formulated as an unconstrained optimization $-I(\vect{V}; \vecthat{Y}) + \lambda I(\vect{X}^N;\vecthat{Y})$, where $\lambda$ is the Lagrange mulitplier and $I$ is the mutual information. Since the distributions involved in the mutual information computation do not have closed forms, we use variational approximations. Let $q(\vect{v}, \vecthat{y})$ and $r(\vecthat{y})$ be the variational approximation of $p(\vect{v}, \vecthat{y})$ and $p(\vecthat{y})$, respectively. Then we get an upper bound on the IB objective of the form
\begin{multline}
	\mkern-18mu -I(\vect{V}; \vecthat{Y}) + \lambda I(\vect{X}^N;\vecthat{Y}) \\
	\mkern-18mu \mkern-18mu \mkern-18mu \mkern-18mu \mkern-18mu \mkern-18mu \mkern-18mu \mkern-18mu \mkern-18mu \leq -\mathbb{E}_{\vect{V},\hatvect{Y}} \left[ \log 	\left( q(\vect{v} \mid \hatvect{y}) \right) \right]  -H(\vect{V}) + \lambda \mathbb{E}_{\vect{X}^N, \hatvect{Y}} \left[ \log 	p(\hatvect{y} \mid \vect{x}^N) \right] - \lambda \mathbb{E}_{\vect{X}^N, \hatvect{Y}} \left[ \log r(\hatvect{y}) \right].
	\label{eqn:vib_simplification}
\end{multline}
Here $H(\vect{V})$ represents the entropy of the random variable $\vect{V}$. Since the noisy channels considered are of the form $\vecthat{y} = h_1(\vect{y}_1,\dots,\vect{y}_N) + \vect{z}$ where $h_1$ is deterministic, $\mathbb{E}_{\vect{X}^N, \hatvect{Y}} \left[ \log 	p(\hatvect{y} \mid \vect{x}^N) \right]$ simplifies to $\mathbb{E}_{\vect{Z}} \left[ \log 	p(\vect{z}) \right]$, which is independent of the encoding and decoding functions. By modeling $q(\vect{v} | \hatvect{y})\propto \exp\left(-\mathcal{D}_{\vect{V}}\left(\vect{v}, g_d\left( \vecthat{y}  \right) \right) \right)$, we can write the minimization objective as
\begin{equation}
	\mathcal{L}_{I} = \mathbb{E}_{\vect{V},\vect{X}^N, \vect{Z}} \left[ \mathcal{D}_{\vect{V}}\left(\vect{v}, g_d\left( h\left( \{ g_e^{(n)}(\vect{x}_n) \}_{n=1}^{N}, \vect{z} \right) \right) \right) \right. -  \left. \lambda\log \left( r\left( h\left( \{ g_e^{(n)}(\vect{x}_n) \}_{n=1}^{N}, \vect{z} \right) \right) \right)\right].
	\label{eqn:ib_loss_func}
\end{equation}

\subsection{Theoretical comparison}
We can show that all the loss functions \eqref{eqn:au_loss_function}, \eqref{eqn:sl_loss_function}, and \eqref{eqn:ib_loss_func} are variational approximations of the Indirect Rate-Distortion problem's minimization objective. In the Indirect Rate-Distortion problem \cite{dobrushin1962information, witsenhausen1980indirect}, a node observes the source $\vect{V}$ through some noisy channel whose output is $\vect{X}^N$. The node uses $\vect{X}^N$ to send a codeword across another rate-limited channel such that the receiver can recover $\vect{V}$. The approximation of $\vect{V}$ at the receiver is $\vecthat{V}$. In the asymptotic case (where the rate $R\leq C$ and $C$ is the channel capacity of the noisy channel), this results in an optimization problem of the form,
\begin{equation}
	\argmin_{p\left(\vecthat{v} \mid \vect{x}^N\right)} \expectation_{\vect{V},\vecthat{V}} \left[\mathcal{D}(\vect{v},\vecthat{v})\right] + \lambda I(\vect{X}^N;\vecthat{V}).
	\label{eqn:IRD_objective}
\end{equation}
Here $\lambda$ is the Lagrange multiplier chosen such that $R \leq C$. One can use \eqref{eqn:IRD_objective} to obtain the optimal encoder and decoder.

The connection between our distributed functional compression framework and the indirect rate distortion problem is presented by \Cref{theo:IB_SL_AU_IRD_objective}. The theorem is derived by first defining two constants $A_1$ and $A_2$ and using the deterministic nature of the encoding and decoding functions. $A_1$ is defined as
\begin{equation}
	A_1 \defeq \sum_{n=1}^N \log (S_{K_n}),
\end{equation}
where $S_{K_n}$ represents the surface area of a $K_n$ dimensional hypersphere with radius $\sqrt{K_n P_T^{(n)}}$. $A_2$ is defined as
\begin{equation}
	A_2 \defeq \sum_{n=1}^N \left(\frac{K_n}{2} \log \left(2\pi \right)+\frac{1}{2} \log \left(K_n P_T^{(n)}\right)\right).
\end{equation}

\begin{theorem}
	If the encoding functions $\{g_e^{(n)}(\cdot)\}_{n=1}^N$ and the decoding function $g_d(\cdot)$ are all deterministic, then for a fixed $\lambda$, the IB based loss function
	\begin{equation}
		\mathbb{E}_{\vect{V},\vect{X}^N, \vect{Z}} \left[ \mathcal{D}_{\vect{V}}\left(\vect{v}, g_d\left( h\left( \{ g_e^{(n)}(\vect{x}_n) \}_{n=1}^{N}, \vect{z} \right) \right) \right) \right. -  \left. \lambda\log \left( r\left( h\left( \{ g_e^{(n)}(\vect{x}_n) \}_{n=1}^{N}, \vect{z} \right) \right) \right)\right] - \lambda H(\vect{Z})
		\label{eqn:theo_IB}
	\end{equation}
	is the variational approximation to a tigher upper bound on \eqref{eqn:IRD_objective}
	than the autoencoder based loss function
	\begin{subequations}
		\begin{equation}
			\mathbb{E}_{\vect{V},\vect{X}^N,\vect{Z}} \left[ \mathcal{D}_{\vect{V}}\left(\vect{v}, g_d\left( h\left( \{ g_e^{(n)}(\vect{x}_n) \}_{n=1}^{N}, \vect{z} \right) \right) \right) \right] + \lambda A_1,
			\label{eqn:theo_AU}
		\end{equation}
		and the lagrangian based loss function
		\begin{equation}
			\mathbb{E}_{\vect{V},\vect{X}^N,\vect{Z}} \left[ \mathcal{D}_{{V}}\left(\vect{v}, g_d\left( h\left( \{ g_e^{(n)}(\vect{x}_n) \}_{n=1}^{N}, \vect{z} \right) \right) \right) \right] + \sum_{n=1}^{N} \frac{\lambda}{K_n P_T^{(n)}}\mathbb{E}_{\vect{X}} \left[ \norm{g_e^i \left(\vect{x}_i \right)}_2^2 \right] + \lambda A_2.
			\label{eqn:theo_SL}
		\end{equation}
	\end{subequations}
	\label{theo:IB_SL_AU_IRD_objective}
\end{theorem}
\vspace{-10mm}
A sketch of the proof is given in \cref{subsec:proof_theo_IB_SL_AU_IRD_objective}. The above theorem conveys that the training objective \eqref{eqn:ib_loss_func} is likely to be a tighter upper bound on the optimal IRD objective than \eqref{eqn:au_loss_function} and \eqref{eqn:sl_loss_function}.

\section{Distributed Training}
The sensor nodes collect the training data in a distributed manner. Transferring the raw sensory data to a central location to train the system can be communication-intensive. Alternatively, we can train over the communication channel as long as the channel is additive. However, it would be suitable to reduce the communication burden during training, as much as possible. To address this we propose two alternative frameworks based on the IB-based loss function \eqref{eqn:ib_loss_func}.

\subsection{Three-stage training of distributed functional compression}
\label{subsec:three_stage_gen}
One way to train the system is to perform end-to-end training over the channel. In this setup, the encoders encode the data and transmit it across the noisy channel to the edge router in the forward pass. In the backward pass, the edge router computes the gradient w.r.t. the loss $\Lloss$ to update $\vect{\Theta}_{N+1}$. To compute the gradients of $\frac{\partial\Lloss}{\partial \vect{\Theta}_n}$ it is sufficient to obtain $\frac{\partial\Lloss}{\partial \vecthat{y}}$ from the edge router. Then, by exploiting the chain rule of differentiation, we can train the encoders.

However, especially during the initial part of the training, the encoder transmissions $g_e^{(n)}(\vect{x}_n)$ are not informative about the input $\vect{x}_n$. Thus, we waste a lot of communication bandwidth in these initial training iterations. By assuming that the functional value for the training dataset is available at all nodes, we propose a novel three-stage training framework to overcome this waste. In the first stage, each sensor node trains the encoder independently without any communication cost. This stage ensures that the transmission from node-$n$ is maximally informative about the function value $\vect{v}$ before any actual communication to the edge router can occur. In the second stage, the edge decoder is trained independently with a one-time communication cost of transmitting the training dataset in the \emph{encoded} form. This stage ensures that the gradients transmitted to the sensor nodes in the later stage are maximally informative about $\Lloss$. In the third stage, the entire system is fine-tuned for optimal performance.

\subsubsection{Stage 1} We can write the objective in this stage by using the information bottleneck principle as
\begin{equation}
	\min -I(\vect{V};\vecthat{Y}_n) + \lambda I(\vect{X}_n,\vecthat{Y}_n).
\end{equation}
This ensures that the transmitted signal retains information about $\vect{V}$ while removing any unnecessary information about $\vect{X}_n$. Similar to the simplifications in \cref{subsec:vib}, we use variational approximations. To approximate the first term, we use a local helper decoder $g_h^{(n)}: \mathcal{\vecthat{Y}}_n \rightarrow \mathcal{\vecthat{V}}$ with parameters $\vect{\Phi}_n$. The objective function for minimization is
\begin{equation}
	\mathcal{L}_{I}^{(n)} = \mathbb{E}_{\vect{X},\vect{V}, \vectcheck{Z}} \left[ \mathcal{D}_{\vect{V}}\left(\vect{v}, g_h^{(n)}\left( g_e^{(n)}(\vect{x}_n)+\vectcheck{z}_n \right) \right) \right. -  \left. \lambda\log \left( r^{(n)}\left( g_e^{(n)}(\vect{x}_n)+\vectcheck{z}_n \right) \right)\right].
	\label{eqn:ib_loss_func_sensorn_stage1}
\end{equation}
Here $\vectcheck{z}_n$ is a simulated noisy channel, and $r^{(n)}(\cdot)$ refers to the variational approximation to the distribution of the simulated noisy received signal at node-$n$.

\emph{Gaussian MAC:} In communication over GMAC, the independent encoder training can lead to a scenario where the superposition of signals will destroy all meaningful information. To overcome this, we propose an idea to embed the transmitted values from different sensor nodes in orthogonal subspaces of $\mathbb{R}^K$. The encoding function decomposes as
\begin{equation}
	g_e^{(n)}(\vect{x}_n) = T_n W_n g_e^{(n,1)}(\vect{x}_n).
\end{equation}
Here $g_e^{(n,1)}(\vect{x}_n) \in \mathbb{R}^{d_n}$, $W_n \in \mathbb{R}^{r \times d_n}$, and $T_n \in \mathbb{R}^{K \times r}$.  We use a fixed orthonormal matrix $T_n$ for embedding the $r$-dimensional vector into a $k$-dimensional space. $T_m^T T_n=\delta(m-n)\vect{I}_r$, where $\delta(0)=1$ and $0$ otherwise. Let us denote $\vect{i}_n = W_n g_e^{(n,1)}(\vect{x}_n)$. Then, $\vecthat{y} = \sum_{n=1}^N T_n \vect{i}_n + \vect{z}$. By the orthonormal property, the edge router can recover a noisy approximation $\vecthat{i}_n = \vect{i}_n + T_n^T \vect{z}$. This structural constraint ensures that there is actionable information at the edge router for stage-2. We remove the constraints in stage-3.

\subsubsection{Stage 2} In this stage, the edge decoder is trained based on the encoded data for the entire dataset transmitted once by the sensor nodes after stage-1. This ensures that the gradients from the decoder in the subsequent stage will be maximally informative about the loss function $\Lloss$. Thus we can write the training loss function as
\begin{equation}
	\mathcal{L}_{I}^{(ER)} = \expectation_{\vecthat{Y},{V}} \left[\mathcal{D}_{{V}}(\vect{v},g_d(\vecthat{y}))\right].
	\label{eqn:loss_func_fusion_center}
\end{equation}

\subsubsection{Stage 3} The training in stage-1 for the encoders is independent, and thus, the encoders learn greedy encoding functions that are maximally informative about $\vect{v}$. In our setup, we are interested in collaboratively computing the target function value. Thus, we perform end-to-end fine-tuning of all the encoders and the decoder using the loss function \eqref{eqn:ib_loss_func}.

\subsection{Distributed training of classifiers}
\label{subsec:distributed_training_classification}
Even though the previous discussion in \cref{subsec:three_stage_gen} reduces the amount of communication during training, we can exploit the channel structure and the classification nature of the function to eliminate the need for any end-to-end iterations completely. Moreover, this algorithm only needs communication once in some $E \gg 1$ iterations, unlike the end-to-end mechanism, which needs to communicate every iteration. We assume that $\mathcal{V}$ is the set of class labels and the class labels are $1,\dots,|\mathcal{V}|$.

We reformulate the optimization problem in stage-3 into a variant of asynchronous block coordinate descent \cite{liu2014asynchronous} as
\begin{subequations}
	\begin{equation}
		\vect{\Theta}_n^* = \argmin_{\vect{\Theta}_n} 	\Lloss_I(\vect{\Theta}_1^*,\dots,\vect{\Theta}_n,\dots,\vect{\Theta}_N^*,\vect{\Theta}_{N+1}^*),
		\label{eqn:enc_stage3a_opt_prob}
	\end{equation}
	\begin{equation}
		\vect{\Theta}_{N+1}^* = \argmin_{\vect{\Theta}_n} 		\Lloss_I(\vect{\Theta}_1^*,\dots,\vect{\Theta}_N^*,\vect{\Theta}_{N+1}).
		\label{eqn:dec_stage3a_opt_prob}
	\end{equation}
\end{subequations}
Here $n \in \{1,\dots,N\}$ and $\Lloss_I$ is defined in \eqref{eqn:ib_loss_func}. We alternate between optimization problems \eqref{eqn:enc_stage3a_opt_prob} and \eqref{eqn:dec_stage3a_opt_prob} till convergence. The equation \eqref{eqn:enc_stage3a_opt_prob} is a set of optimization problems that is carried out asynchronously at the sensor nodes. In the following we provide a methodology for approximating the loss function for \eqref{eqn:enc_stage3a_opt_prob} without exchanging neural network parameters. The optimization in \eqref{eqn:dec_stage3a_opt_prob} can be carried out in a similar manner to stage-2.

\subsubsection{Gaussian MAC}
\label{subsubsec:gmac_stage3} 
To understand how to approximate the loss function locally, let us revisit \eqref{eqn:vib_simplification}. We can modify it as
\begin{multline}
	-I({V}; \vecthat{Y}) + \lambda I(\vect{X}^N;\vecthat{Y}) \\
	\leq -\mathbb{E}_{\vect{V},\hatvect{Y}} \left[ \log \left( q_h^{(n)}({v}| \hatvect{y}) \right) \right] + \expectation_{{V},\hatvect{Y}} \left[\log \left( \frac{q_h^{(n)}({v}| \hatvect{y})}{q_d({v}|\hatvect{y})} \right)\right] - \lambda \mathbb{E}_{\vect{X}^N, \hatvect{Y}} \left[ \log r(\hatvect{y}) \right] + A_3.
	\label{eqn:vib_simplification_again}
\end{multline}
Here $A_3$ represents all the constant terms in \eqref{eqn:vib_simplification}, $\vect{q}_d({v}|\hatvect{y})$ represents the probability of class-$v$ as predicted by the edge decoder, and $\vect{q}_h^{(n)}({v}|\hatvect{y})$ represents the same but as predicted by a local helper decoder at node-$n$. The second term is minimized when the predictive probability of the correct class labels from both the helper decoder and the edge decoder match. This is very similar to the knowledge distillation problem where a teacher classifier helps guide the training of a student classifier \cite{hinton2015distilling}. Further, knowledge distillation has shown excellent results by exploiting the \emph{dark knowledge} in the classifier output \cite{xu2018interpreting}. This dark knowledge refers to the implicit information contained in the predictive distribution of a classifier. For example, an image classified as a boat with probability $0.9$ and as a car with probability $0.1$ has to be encoded differently from an image that is classified as a boat with probability $0.9$ and plane $0.1$. So we modify the training loss function and replace the second term in \eqref{eqn:vib_simplification_again} with $\beta D_{KL}(\vect{q}_h^{(n)}(w| \hatvect{y}; \gamma)||\vect{q}_d(w|\hatvect{y}); \gamma)$, where $\beta$ is some weighting factor. Here $\vect{q}_d(w|\vecthat{y})$ represents the output distribution over the set of class labels $\mathcal{V}$ from the edge decoder, $\gamma$ is the temperature used in the softmax function \cite{hinton2015distilling}, and $D_{KL}$ is the KL divergence term.

\begin{figure}[h]
	\centering
	\includegraphics[width=0.9\linewidth]{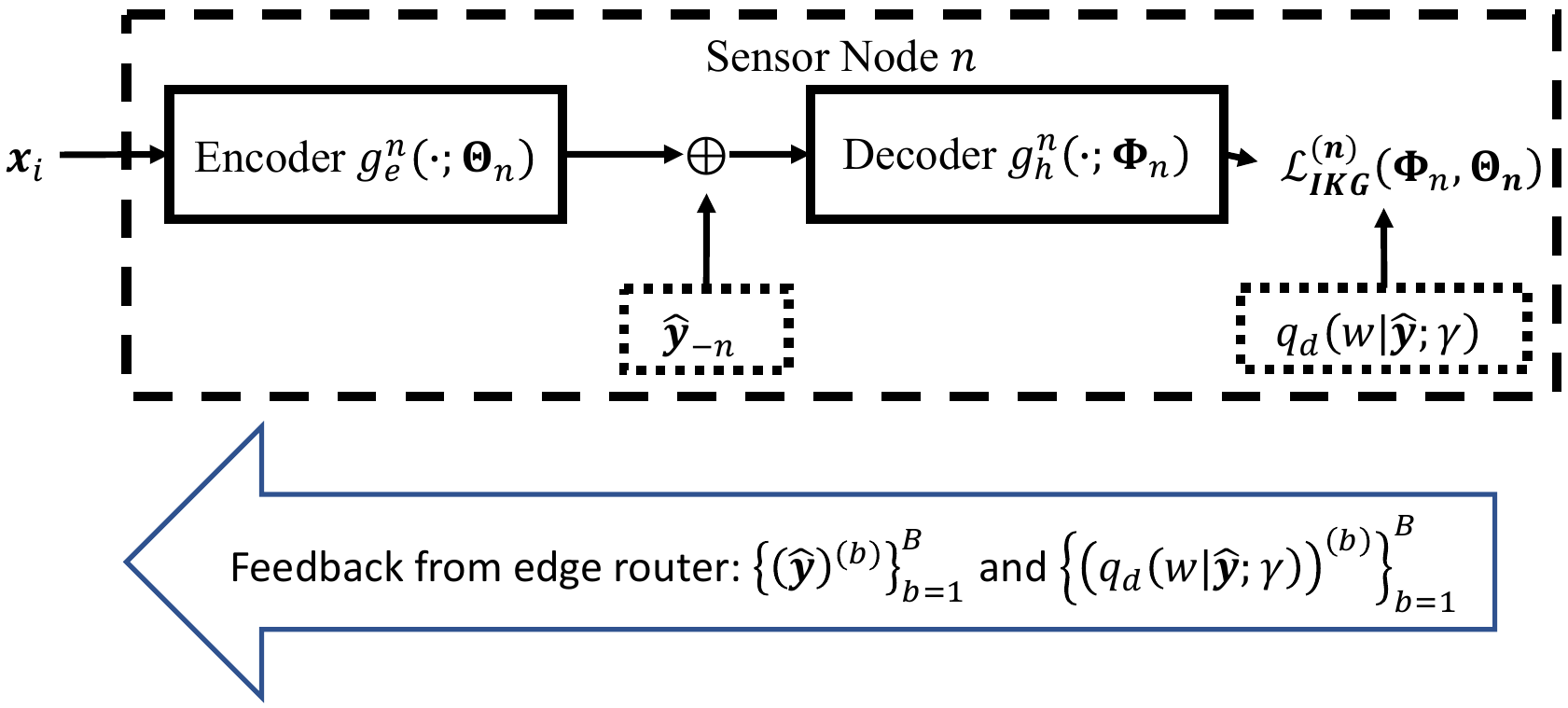}
	\caption{Stage 3: Training the encoder at the sensor node-$n$ for distributed classification over GMAC.}
	\label{fig:gmac_stage3_classifier}
%	\vspace{-7mm}
\end{figure}
\Cref{fig:gmac_stage3_classifier} shows the training of the encoder and the helper decoder at node-$n$. We represent the feedback as $\{(\cdot)^{(b)}\} _{b=1}^{B}$, which denotes that the value of the term $(\cdot)$ is computed and sent as feedback for every example in the training dataset. We define $\vecthat{y}_{-n} \defeq \vecthat{y} - \vect{y}_n$. Note that the feedback happens every $E\gg 1$ iterations, and $\vecthat{y}_{-n}$ is held constant during that time. $r^{(n)}(\vecthat{y})$ is the local variational approximation of $p(\vecthat{y})$ at node-$n$. Thus we can write the final training loss function as
\begin{multline}
	\mathcal{L}_{IKG}^{(n)} = \mathbb{E}_{\vect{X}_n,{V}, \vecthat{Y}_{-n}} \left[ \mathcal{D}_{{V}}\left({v}, g_h^{(n)}\left( g_e^{(n)}(\vect{x}_n) + \vecthat{y}_{-n} \right) \right) \right. \\
	-  \left. \lambda\log \left( r^{(n)}\left( g_e^{(n)}(\vect{x}_n) + \vecthat{y}_{-n} \right) \right)\right.\\
	+ \left. \beta D_{KL}\left(\vect{q}^{(n)}_h(w|\vect{\tilde{y}}_n; \gamma) \mid \mid \vect{q}_d(w|\vecthat{y}; \gamma) \right) \right].
	\label{eqn:ikg_loss_func}
\end{multline}

\subsubsection{Orthogonal AWGN channel using Product of Experts}
\label{subsubsec:awgn_stage3}
 In AWGN channels, the edge router receives each encoder's transmission independently. The noisy received signal from node-$n$ is denoted as $\vecthat{y}_n$. Instead of using a standard neural network decoder whose input is $\vecthat{y}$, we use a Product of Experts (PoE) based decoder \cite{hinton2002training} which processes each $\vecthat{y}_n$ separately. This, as we shall show in \cref{subsubsec:awgn_sensor_dropout} provides great performance benefit during sensor outage with no loss in performance compared to a standard decoder.

In PoE, we assume that $\vect{q}_d(w|\vecthat{y}) \defeq\frac{1}{Z(\vecthat{y})} \prod_{n=1}\vect{q}^{(n)}_d({w}|\vecthat{y}_n)$ where $Z(\vecthat{y}) = \sum_{w \in \mathcal{V}} \prod_{n=1}\vect{q}^{(n)}_d({w}|\vecthat{y}_n)$. Using similar derivation as in \eqref{eqn:vib_simplification_again} we can show that
\begin{multline}
	\mkern-18mu -I({V}; \vecthat{Y}) + \lambda I(\vect{X}^N;\vecthat{Y}) \\
	\leq -\sum_{n=1}^{N} \left( \mathbb{E}_{{V},\hatvect{Y}_n} \left[ \log \left( q_h^{(n)}({v}| \hatvect{y}_n) \right) \right] \right. - \left. \mathbb{E}_{{V},\hatvect{Y}_n} \left[  \log \left( \frac{q_h^{(n)}({v}| \hatvect{y}_n)}{q_d^{(n)}({v}|\hatvect{y}_n)} \right) \right] + \lambda \mathbb{E}_{\vect{X}_n, \hatvect{Y}_n} \left[ \log r^{(n)}(\hatvect{y}_n) \right] \right)\\
	+\mathbb{E}_{\hatvect{Y}} \left[\log(Z(\vecthat{y}))\right] + A_3.
	\label{eqn:vib_simplification_awgn}
\end{multline}
Here $r^{(n)}(\vecthat{y}_n)$ is the local variational approximation for $p(\vecthat{y}_n)$. The summation term is of the form $\sum_n \Lloss_n$, and the gradients of $\vect{\Theta}_n$ only depend on the component $\Lloss_n$.
Unfortunately, due to the $\sum_{w \in \mathcal{V}}$ inside the $\log$ in the $\mathbb{E}_{\hatvect{Y}} \left[\log(Z(\vecthat{y}))\right]$ term, it cannot be decomposed. Instead, we use a local approximation of the form,
\begin{equation}
	Z_n = \sum_{w \in \mathcal{V}} \vect{q}^{(n)}_h({w}|\vecthat{y}_n) \prod_{m \neq n}\vect{q}^{(m)}_d({w}|\vecthat{y}_m).
	\label{eqn:compute_z_n}
\end{equation}
The product term represents the contributions of all other sensor nodes, other than node-$n$, to the output distribution. The term $Z_n$ ensures that the training of encoders is collaborative and accounts for the contribution from the other sensors' encoder. Thus using knowledge distillation, we can write the training loss function as
\begin{multline}
	\mathcal{L}_{IKA}^{(n)} = \mathbb{E}_{\vect{X}_n,{V}, \vectcheck{Z}_{n}} \left[ \mathcal{D}_{{V}}\left({v}, g_h^{(n)}\left( g_e^{(n)}(\vect{x}_n) + \vectcheck{z}_{n} \right) \right) \right. \left. - \log(Z_n) - \lambda\log \left( r\left( \vecthat{y}_n\right) \right)\right.\\
	+ \left. \beta D_{KL}\left(q_h^{(n)}(w| \hatvect{y}_n; \gamma) || q_d^{(n)}(w|\hatvect{y}_n; \gamma) \right) \right],
	\label{eqn:ika_loss_func}
\end{multline}
where $\vectcheck{z}_n$ is a simulated channel. \Cref{fig:awgn_stage3_classifier} shows the training of the encoder and helper decoder at sensor node-$n$. The feedback happens every $E\gg 1$ iterations.

\begin{figure}[h]
	\centering
	\includegraphics[width=0.9\linewidth]{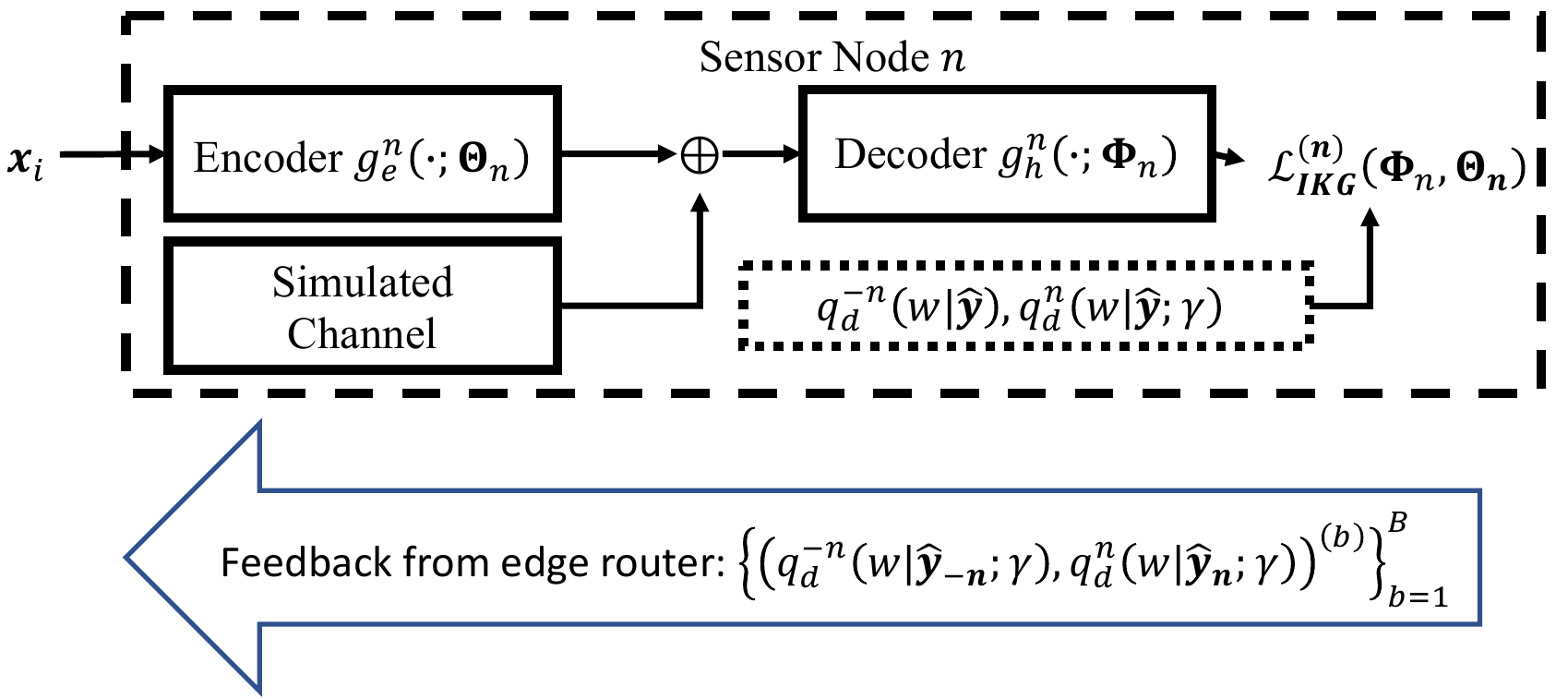}
	\caption{Stage 3: Training the encoder at the sensor node-$n$ for distributed classification over AWGN.}
	\label{fig:awgn_stage3_classifier}
	\vspace{-7mm}
\end{figure}

\subsection{Convergence Results for distributed training}

\subsubsection{Noiseless Gradient Approximation}

We use $n \in \{1,\dots,N+1\}$ to denote the nodes, with $N+1$ denoting the edge router. We assume that each node-$n$ updates its parameter block $\vect{\Theta}_n$ by performing gradient descent using $\nabla_{\vect{\Theta}_n} \Lloss (\vect{\Theta}_1,\dots,\vect{\Theta}_{N+1})$. To facilitate the computation of this gradient, we assume a modified setup where every node-$n$ broadcasts its latest parameters to the other nodes after $E$ steps of gradient descent. Let $\paramstep{\vect{\Theta}}{s} \defeq [ \paramstep{\vect{\Theta}_1}{s},\dots,\paramstep{\vect{\Theta}_{N+1}}{s} ]$ denote the vector containing the latest parameters from all nodes at some step $s$. The step number counts the total number of parameter updates across all parameter blocks so far. By definition, an increment of $s$ corresponds to an update of one block of parameters. Let us define a function $r(s,n)$ that returns the most recent step $\leq s$ that corresponds to an update of the parameter block $\vect{\Theta}_n$. Define $s' \defeq \floor{\frac{s}{(N+1)E}}(N+1)E$ which denotes the iteration number corresponding to the most recent exchange of parameters between the nodes. Then the current copy of parameters at node-$n$ can be written as $\paramstep{\hatvect{\Theta}}{s,n} = \left[\paramstep{\vect{\Theta}_1}{r(s',1)},\dots,\paramstep{\vect{\Theta}_n}{r(s,n)},\dots,\paramstep{\vect{\Theta}}{r(s',N+1)}\right]$. If block-$n$ was the parameter block updated at step $s$, we can write the global behavior of the algorithm as
\begin{equation}
	\paramstep{\vect{\Theta}}{s+1} = \paramstep{\vect{\Theta}}{s}  - \frac{\eta}{\mathrm{L}} \nabla_{n} \mathcal{L}\left( \paramstep{\hatvect{\Theta}}{s,n} \right).
\end{equation}
Here $\nabla_{n} \mathcal{L} \defeq \left[\vect{0}, \dots, \nabla_{\vect{\Theta}_{n}} \mathcal{L}, \dots, \vect{0}\right]^T$. Since the largest difference between $s$ and $s'$ is $(N+1)E$ and the largest difference between $s'$ and $r(s',\cdot)$ is $NE$, we can conclude that the largest difference between $s$ and $s'$ is $\leq (2N+1)E$. This implies that no node is computing gradients using a parameter block that is more than $(2N+1)E$ steps older than the current step. We denote this bound on the age of the parameter block as $\tau$.

We make the following assumptions.
\begin{itemize}
	\item \textbf{Assumption 1.} The loss function $\mathcal{L}$ is assumed to be non-convex, has $\mathrm{L}$-Lipschitz gradients, and a finite minimum. We define the minima as $\Lloss^*$.
	\item \textbf{Assumption 2.} Let the learning rate be chosen that it satisfies $\eta < \frac{2}{2\tau+1}$. Define $\alpha \defeq \frac{2}{\eta \left(2\tau+1 \right)} > 1$.
\end{itemize}
%\begin{asu}
%	The loss function $\mathcal{L}$ is assumed to be non-convex, has $\mathrm{L}$-Lipschitz gradients, and a finite minimum. We define the minima as $\Lloss^*$.
%	\label{asu-lipschitz-grad-finite-min-loss}
%\end{asu}
%
%\begin{asu}
%	Let the learning rate be chosen that it satisfies $\eta < \frac{2}{2\tau+1}$. Define $\alpha \defeq \frac{2}{\eta \left(2\tau+1 \right)} > 1$.
%	\label{asu-learning-rate-small}
%\end{asu}
%\vspace{-7mm}

\begin{theorem}
	If $\mathcal{L}$ satisifies Assumption 1 and the learning rate satisfies Assumption 2, then
	\begin{multline}
		\frac{1}{S} \sum_{s=1}^{S} \norm{\nabla \mathcal{L}\left(\paramstep{\vect{\Theta}}{k}\right)}_2^2 \\
		< \frac{\Lloss^{(1)}-\Lloss^*}{S} \left( \frac{\mathrm{L} \alpha^2 (2N+1)E \left(4(N+1)E+1\right)}{\mu^2 \left(\alpha-1\right)} \right. + \left. \frac{2 \mathrm{L} (N+1)(2N+1)E \left(4(2N+1)E-1\right)}{\left(1-\mu^2\right)\left(\alpha-1\right)\left(2(N+1)E +\frac{1}{2}\right)} \right).
		\label{eqn:theo-grad-avg-upperbound-noiseless}
	\end{multline}
	Here, $\mu^2 \in \left(0,1\right)$ and $\Lloss^{(1)}$ is the value of the loss function computed at the initialization point. Further,
	\begin{equation}
		\lim_{s \rightarrow \infty} \norm{\nabla \mathcal{L}\left(\paramstep{\vect{\Theta}}{s}\right)}_2^2 = 0.
	\end{equation}
	\label{theo-grad-avg-upperbound-noiseless}
\end{theorem}
\vspace{-15mm}
We provide a sketch of the proof for \cref{theo-grad-avg-upperbound-noiseless} in \cref{subsec:proof-theo-grad-avg-upperbound-noiseless}. This theorem shows that the convergence rate is upper bound by the square of the number of sensor nodes. It also shows that the algorithm converges to a stationary point of $\Lloss$. \Cref{remark:comm_rounds_noiseless} shows an upper bound on the number of communication rounds required.
\begin{remark}
	Let $S=(N+1)ET+S'$ where $T \defeq \floor{\frac{S}{(N+1)E}}$. Let $\min_{k=1,\dots,K} \norm{\nabla \mathcal{L}\left(\paramstep{\vect{\Theta}}{k}\right)}_2^2 = \delta$. Then, from \Cref{theo-grad-avg-upperbound-noiseless} $T < \mathcal{O}\left(\left(\Lloss^{(1)}-\Lloss^*\right)\frac{N^2}{\delta}\right)$.
	\label{remark:comm_rounds_noiseless}
\end{remark}

\subsubsection{Noisy Gradient Approximation}

In the previous theorem, we assumed that each node could compute gradients of the loss function, albeit with an older set of parameters. In this subsection we assume that the the gradients from the approximated loss function are an unbiased but noisy estimate of the actual loss gradients, where the noise has bounded variance.
%We have to transmit actual parameters to the different nodes in order to facilitate this. In reality, we employ approximated loss function values to compute the gradients. We modify our setup to close the gap between the theoretical convergence analysis and actual implementation. We assume that the gradients from the approximated loss function are an unbiased but noisy estimate of the actual loss gradients, where the noise has bounded variance.

Formally, the global parameter update is carried out in the following fashion,
\begin{equation}
	\paramstep{\vect{\Theta}}{s+1} = \paramstep{\vect{\Theta}}{s}  - \frac{\eta}{\mathrm{L}} \widetilde{\nabla}_{n} \mathcal{L}\left( \paramstep{\hatvect{\Theta}}{s,n} \right).
\end{equation}
Here $\widetilde{\nabla}_{n} \mathcal{L} \defeq {\nabla}_{n} \mathcal{L} + \vect{\epsilon}^{(s)}_n$ and $\vect{\epsilon}^{(k)}_l$ noise. Note that $\vect{\epsilon}^{(k)}_l$ is zero for all indices corresponding to parameters that are not in block-$n$, and the $s$ is used to indicate the parameter step at which it is acting. If we directly share all the parameters with all the sensors and the edge router, then $\vect{\epsilon}_l=\vect{0}$. Instead, we share some other information as feedback. We assume this feedback results in noisy but unbiased gradients at the individual nodes where the noise has bounded variance. Overall, we make the following assumptions
\begin{itemize}
	 \item \textbf{Assumption 3.} $\expectation \left[\widetilde{\nabla}_{l} \mathcal{L} \right] = \nabla_{l} \mathcal{L}$ and  $\expectation \left[ \norm{\widetilde{\nabla}_{l} \mathcal{L} - \nabla_{l} \mathcal{L}}^2_2\right] = \sigma^2_l < \infty$.
	 \item \textbf{Assumption 4.} The noise is independent across both parameter blocks and parameter steps, i.e., $\vect{\epsilon}^{(i)}_l \independent \vect{\epsilon}^{(j)}_m \, \forall i \neq j \text{ and } l,m \in \{1,\dots,N+1\}$.
	 \item \textbf{Assumption 5.} The learning rate is assumed to be $\eta \leq \frac{1}{\tau+1}$. Let $\beta \defeq \eta (\tau+1)$.
\end{itemize}

\begin{theorem}
	If $\mathcal{L}$ satisifies Assumption 2, the noise satsifies Assumption 3, and Assumption 4, and the learning rate satisfies Assumption 5, then
	\begin{multline}
		\frac{1}{S}\sum_{s=1}^{S} \expectation \left[\norm{\nabla \Lloss \left(\paramstep{\vect{\Theta}}{s}\right)}_2^2\right] \\
		< \frac{\Lloss^{(1)}-\Lloss^*}{S} \left( \frac{\tau \mathrm{L}}{\mu^2 \eta(1-\eta(\tau+1))} + \frac{2 \eta \mathrm{L} (N+1) (2\tau + 1)(4\tau - 1)}{(1-\mu^2)(4 - \eta (3\tau+4))} \right)
		\\
		+ \eta \sigma^2 \left( \frac{\left( \tau + 2 \right)}{2 \mu^2 (N+1)(1-\eta(\tau+1))} + \frac{2 \eta (N+1) (2\tau + 1)}{(1-\mu^2)(4 - \eta (3\tau+4))} \left( \frac{8 E\left(\tau -\frac{1}{2}\right)}{S} + 4 E \right) \right).
		\label{eqn:theo-grad-avg-upperbound-noisy-v2}
	\end{multline}
	Here, $\mu^2 \in \left(0,1\right)$, $N+1$ is the number of nodes, $E$ is the number of local iterations between parameter exchanges, $L^{(1)}$ is the loss at the initial point, and $\Lloss^*$ is the minimum loss possible.
	\label{theo-grad-avg-upperbound-noisy-v2}
\end{theorem}
Sketch of the proof is provided in \cref{sec:proof-theo-grad-avg-upperbound-noisy}. Similar to \Cref{theo-grad-avg-upperbound-noiseless}, \Cref{theo-grad-avg-upperbound-noisy-v2} also shows that the convergence rate is upper bound by the square of the number of sensor nodes. Similar to results corresponding to the convergence of Stochastic Gradient Descent, to reduce the effect of noise at convergence, it is necessary to reduce the learning rate. For small learning rates, we get the same convergence results as \Cref{theo-grad-avg-upperbound-noiseless}.

\section{Experiments}
\subsection{Experimental settings}
\begin{figure}[htbp]
	\begin{center}
		\begin{subfigure}[b]{0.49\textwidth}
			\centering
			\includegraphics[width=0.6\linewidth]{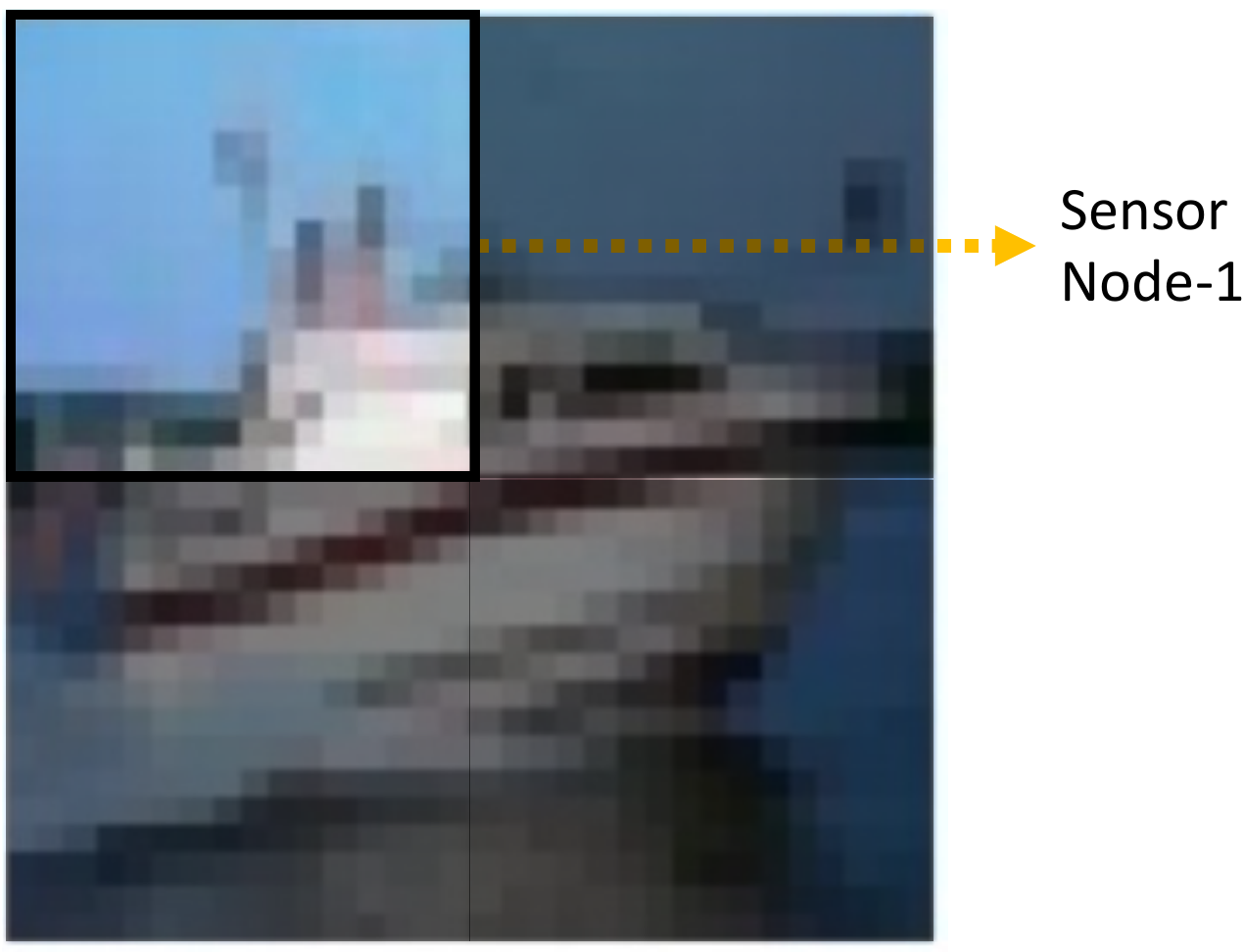}
			\caption{Sensor node-$1$}
			\vspace{-3mm}
			\label{fig:sensor1of9}
		\end{subfigure} \hfill
		\begin{subfigure}[b]{0.49\textwidth}
			\centering
			\includegraphics[width=0.6\linewidth]{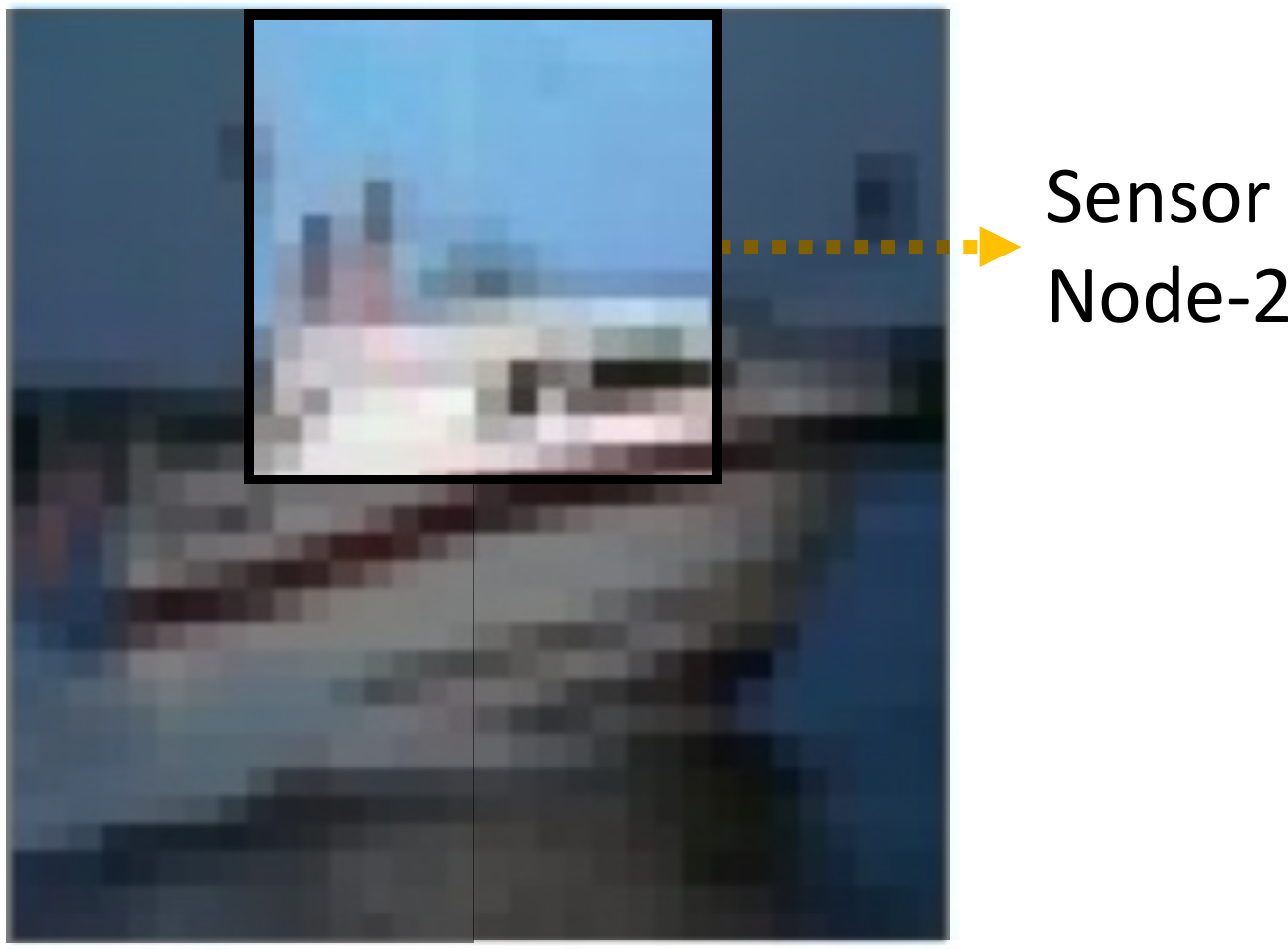}
			\caption{Sensor node-$2$}
			\vspace{-3mm}
			\label{fig:sensor2of9}
		\end{subfigure}
	\end{center}
	\vspace{-5mm}
	\caption{Inputs shown to sensor nodes in a $N=9$ setup.}
	\label{fig:cifar_setup_intersect_9}
	\vspace{-7mm}
\end{figure}

We use the CIFAR-10 dataset to perform experiments \cite{krizhevsky2009learning}. When $N=4$ nodes, each node observes a unique quadrant of the image, with the first node observing the top left quadrant, the second observing the top right quadrant, and so on. As the number of encoders increases, the $16 \times 16$ image patches are allowed to intersect, as shown in \Cref{fig:cifar_setup_intersect_9}. Our objective is to recover the classification label.

\begin{table*}[t]
	\centering
	\caption{Architecture of the DNNs used for the CIFAR10 dataset.}
	\vspace{-2mm}
	\begin{tabular}{|c||c|}
		\hline
		Name & Architecture Details \\
		\hline
		\hline
		VGG\_BLK($F$) & [Conv($F$,$3\times3$),Conv($F$,$3\times3$),MaxPool($2 \times 2$)]\\
		\hline
		Encoder & [VGG\_BLK($64$),VGG\_BLK($128$),VGG\_BLK($512$),VGG\_BLK($512$),FCN($1024$),FCN($512$),FCN($K$)] \\
		\hline
		Decoder & [FCN($512$),FCN($1024$),FCN($2048$),FCN($128$),FCN($10$)] \\
		\hline
	\end{tabular}
	\label{tab:arch_dets}
	\vspace{-5mm}
\end{table*}

\Cref{tab:arch_dets} shows the architectural details of the encoders and the decoder used. We use the Adam optimizer \cite{kingma2014adam} and regardless of the initial learning rate, we decay the learning rate by $0.5$ when the validation loss saturates. In the end-to-end training, stage-1, and stage-2 of the three-stage training, the initial learning rate is $10^{-3}$. In stage-3, the initial learning rate is $2.5 \times 10^{-4}$. We use 45000 images for training, 5000 for validation, and 10000 for testing with a batch size of 64. The classification accuracy presented here corresponds to ten repetitions over the test set. For all experiments, we assume $P_T^{(n)}=P_T$, and in the case of \eqref{eqn:sl_loss_function}, $\lambda_n=\lambda \, \forall n \in \{1,\dots,N\}$. We model the distribution of the received $\vecthat{y}$ as a product of independent Gaussian distributions for every dimension. Each component has zero mean and common variance. We fold the variance learning into the $\lambda$ selection problem. We set $E=30$ epochs, $\beta=0.1$, $\gamma=2$ for AWGN, and $\gamma=3$ for GMAC. 

Finally, to avoid the growing number of decoder expert networks at the edge router as $N$ increases in the PoE setup in \cref{subsubsec:awgn_stage3},  we implement $g_d^{(n)}(\cdot) = g_{dc}(\cdot,n)$, where $g_{dc}(\cdot)$ is a common network which gets an additional input $n$. Further, for the orthogonal AWGN channels  we can use only one feedback $\vect{q}_d(w|\vecthat{y})$ to approximate the other two \cref{subsubsec:awgn_stage3}, thus further reducing communication.

\subsection{Simulation results for orthogonal AWGN channels}
\subsubsection{Varying the channel capacity}
\begin{table}[h]
	\centering
	\caption{Classification performance over AWGN Channels for $N=4$ and various channel capacities $C$}
	\begin{tabular}{|c||c|c|c|c|c|}
		\hline
		Method & C=$12$ & C=$16$ & C=$20$  \\
		\hline
		\hline
		JPEG2000+Classifier (P2P) & 10\% & 10\% & 10\% \\
		\hline
		JSCC \cite{saidutta2021joint} +Classifier (P2P) & 32.87\% & 36.04\% & 39.03\% \\
		\hline
		IB (P2P) & 90.24\% & 90.25\% & 90.60\% \\
		\hline
		\hline
		NN-REG \cite{hanna2020distributed} & 68.07\% & 73.43\% & 78.12\% \\		
		\hline
		NN-GBI \cite{hanna2020distributed} & 65.16\% & 71.57\% & 81.18\% \\		
		\hline
		\hline
		Autoencoder & 69.55\% & 71.4\%3 & 73.74\% \\
		\hline
		Lagrange Method &  79.62\% & 80.97\% & 81.79\% \\
		\hline
		Information Bottleneck & 79.89\% & 81.79\% & 83.23\% \\
		\hline
	\end{tabular}
	\label{tab:awgn_varying_channel_cap}
	\vspace{-5mm}
\end{table}

We first study the performance of the system over varying channel conditions. $C$ represents the total capacity from all sensors to the edge router. \Cref{tab:awgn_varying_channel_cap} shows the performance of the system. The first three methods are point-to-point (P2P), i.e., the entire image is observed by one sensor. In the JPEG200 based scheme, we compress the image with lossy compression and use a capacity achieving code to transmit the compressed data. Since even at the highest capacity the compression ratio is $1228.8$, this scheme breaks down. In the second row we use a machine learning based Joint Source-Channel Coding scheme \cite{saidutta2021joint} followed by a classifier. Even though this system is better than the JPEG2000 based system, its performance is poor compared to the functional compression schemes showcased from the fourth row onwards.  In the third row the performance of a P2P system trained using the Information Bottleneck principle is shown. The P2P setup models a scenario where the distributed sensors are allowed to coordinate with each other prior to their communication with the edge router. Since our problem setup does not allow such coordination, the IB (P2P) results are an upper bound to the performance of a distributed setup. The other two comparison methods are got from \cite{hanna2020distributed}, who also assume a distributed setup similar to ours for quantization. However, since their setup is digital, we assume that they are operating along with a short length capacity achieving code. Our Information Bottleneck-based training \eqref{eqn:ib_loss_func} outperforms all other benchmarks (except the upper bound) and is the best amongst the three training loss functions presented in \cref{sec:tale_of_three_loss_functions}.

\Cref{fig:awgn_robustness} shows the robustness of the learned encoders and decoders by varying the channel conditions. We use the Information Bottleneck-based loss function for training. $C_{tr}$ is the channel capacity assumed while training and $C_{te}$ the  capacity at test time. We see that even when $|C_{tr}-C_{te}|=8$ bits, the performance loss is only around $1\%$.

\begin{minipage}{.47\textwidth}
	\centering
	\includegraphics[width=.8\linewidth]{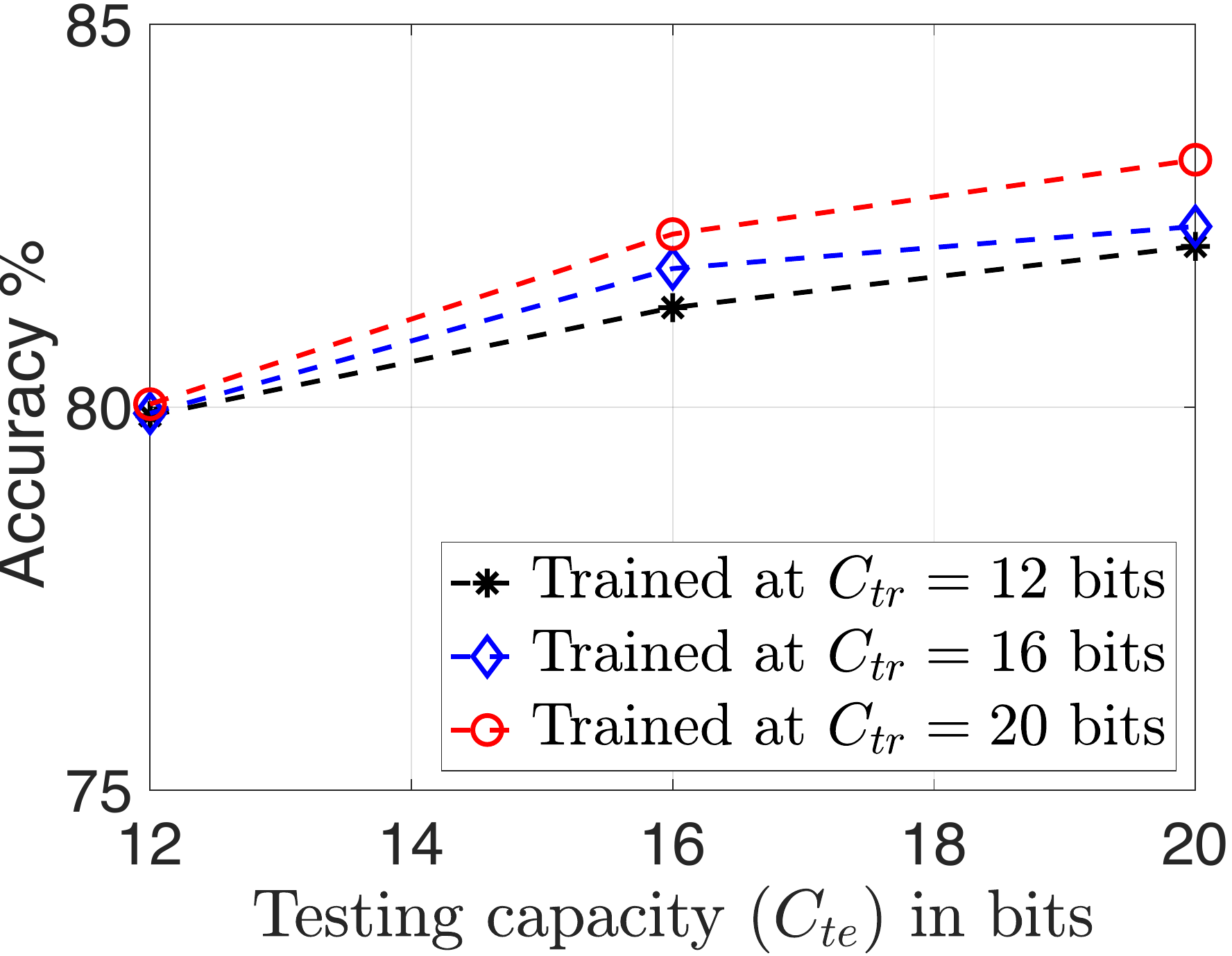}
	\captionof{figure}{Performance when $C_{tr}\neq C_{te}$ over AWGN.}
	\label{fig:awgn_robustness}
\end{minipage}\hfill
\begin{minipage}{.47\textwidth}
	\centering
	\includegraphics[width=.8\linewidth]{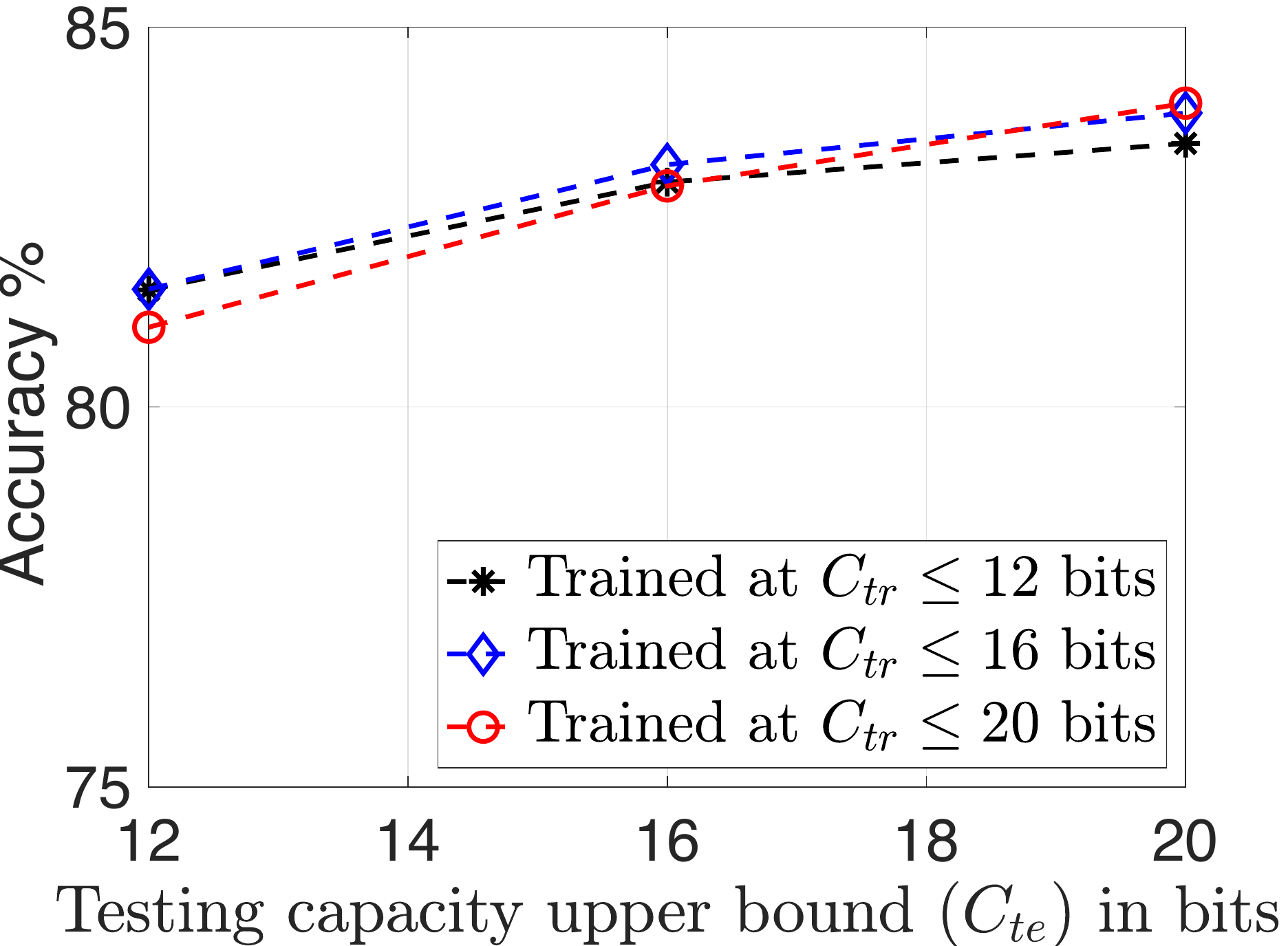}
	\captionof{figure}{Performance when $C_{tr}\neq C_{te}$ over GMAC.}
	\label{fig:gmac_robustness}
\end{minipage}
\vspace{-5mm}

\subsubsection{Varying the number of sensors}
\label{subsubsec:awgn_varying_number_sensors}
\begin{table}[h]
	\centering
	\caption{AWGN Performance at $C=20$ and varying $N$.}
	\begin{tabular}{|c||c|c|c|c|c|}
		\hline
		Method & $N=4$ & $N=9$ & $N=16$  \\
		\hline
		\hline
		IB+3S & 83.23\% & 83.43\% & 83.36\% \\
		\hline
		IB+3S+E2E & 83.44\% & 83.26\% & 82.80\%* \\
		\hline
	\end{tabular}
	\label{tab:awgn_varying_sensors}
%	\vspace{-3mm}
\end{table}
\Cref{tab:awgn_varying_sensors} shows the performance of the three-stage learning schemes when the number of sensors increases. The IB+3S refers to the three-stage training scheme when the third stage uses the method described in \cref{subsubsec:awgn_stage3}, IB+3S+E2E refers to the training scheme described in \cref{subsec:three_stage_gen}. Both schemes use the Information Bottleneck-based loss function. We fix the total channel capacity from all encoders to the edge router to $C=20$ bits, with each sensor getting $C/N$ bits. Both the algorithms have similar performance and scale well with $N$. For $N=16$, the IB+3S+E2E scheme had not converged even after 500 epochs.

\begin{table}[h]
	\centering
	\caption{Channel uses for varying $N$ over AWGN.}
	\label{tab:awgn_channel_use_varying_sensors}
	\begin{tabular}{|c|c||c|c|c|}
		\hline
		\multicolumn{2}{|c||}{Method} & CT & IB+3S+E2E & IB+3S \\ 
		\hline
		\hline
		\multirow{2}{*}{$N=4$}  & S$\rightarrow$R  & $7.3e8$  & $4.6e7$  & $3.2e6$ \\ \cline{2-5} 
		& R$\rightarrow$S  & $3.6e7$   & $6.4e5$ & $2.0e6$  \\ 
		\hline \hline
		\multirow{2}{*}{$N=9$}  & S$\rightarrow$R  & $1.9e9$ & $2.0e8$ & $8.0e6$  \\ \cline{2-5} 
		& R$\rightarrow$S  & $8.1e7$  & $2.8e6$ & $4.5e6$ \\
		\hline \hline
		\multirow{2}{*}{$N=16$}  &  S$\rightarrow$R & $6.0e9$ &  $>8.0e8$ & $1.9e7$ \\ \cline{2-5} 
		&  R$\rightarrow$S & $1.4e8$ & $>1.1e7$ & $6.0e6$ \\
		\hline
	\end{tabular}
	\vspace{-5mm}
\end{table}

We compare the total number of channel uses across all nodes in \Cref{tab:awgn_channel_use_varying_sensors}. ``S" represents the sensor node and ``R" the edge router. In the Cloud Training (CT) setup, we use lossless JPEG2000 and a capacity-achieving codeto transmit the training data from the sensor to the edge router. We train the complete system at the edge router using the sensory data and transmit the trained weights to the sensor nodes. Each sensor uses $K$ channels to convey $C/N$ bits of information (on average) per data sample where $K=4,2,2$ for $N=4,9,16$ respectively. The sensor to router channel is called uplink, and the reverse is called the downlink. The downlink is assumed to be a $32$-bit capacity channel operating with a capacity-achieving code. The number of channel uses for CT in the uplink is computed using the compressed data size. The downlink usage is computed using the number of parameters of the neural network model and assuming a 32-bit representation for each parameter. The uplink in both the IB-based methods consumes $KBT$ channel uses, where $T$ varies for the two methods. We compute the downlink communication for IB+3S+E2E as $KBT/64$, assuming a 32-bit representation for each gradient value. Similarly, the downlink for IB+3S is used $BT|\mathcal{V}|$ times. Compared to the CT scheme, both our methods show a significant reduction in communication. However, amongst the two, IB+3S is the most impressive, with at least two orders of magnitude reduction in the number of channel uses for the uplink and four orders of magnitude in the downlink. Since we can send the gradients averaged across data points in a batch, the downlink communication in IB+3S+E2E is more efficient. However, as $N$ increases, the increasing number of communication rounds ($T$) renders the IB+3S more efficient even in the uplink. Additionally, the number of communication rounds for IB+3S in the third stage is $T=3,8,11$ for $N=4,9,16$, respectively. This closely follows the relation predicted by \cref{remark:comm_rounds_noiseless}.

\subsubsection{Sensor outages during testing}
\label{subsubsec:awgn_sensor_dropout}
\begin{figure}[htbp]
	%\vspace{-3mm}	
	\begin{center}
		\begin{subfigure}[b]{0.49\textwidth}
			\centering
			\includegraphics[width=0.7\linewidth]{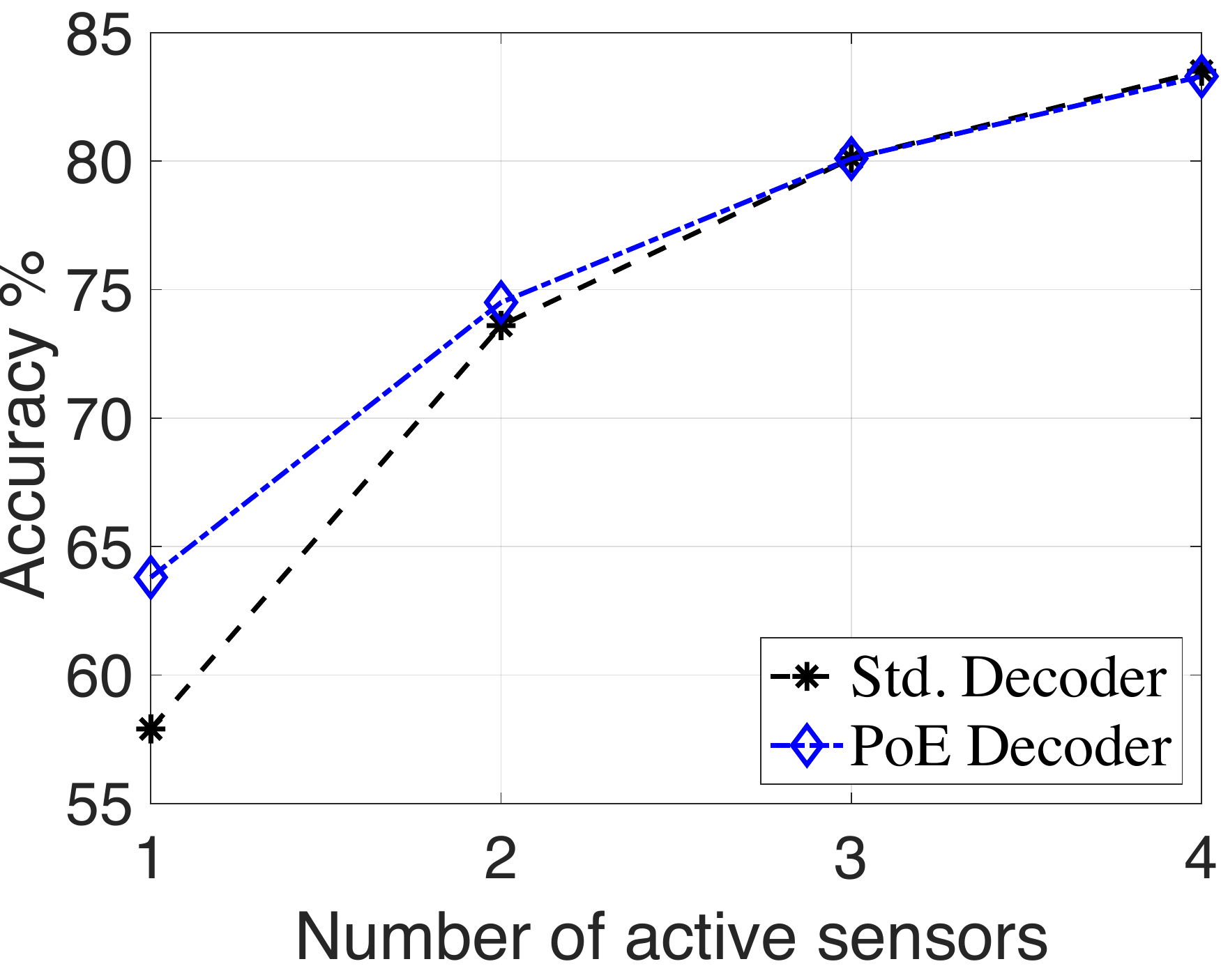}
			\caption{$N=4$ sensors.}
			\label{fig:awgn_test_dropout_N4}
		\end{subfigure} \hfill
		\begin{subfigure}[b]{0.49\textwidth}
			\centering
			\includegraphics[width=0.7\linewidth]{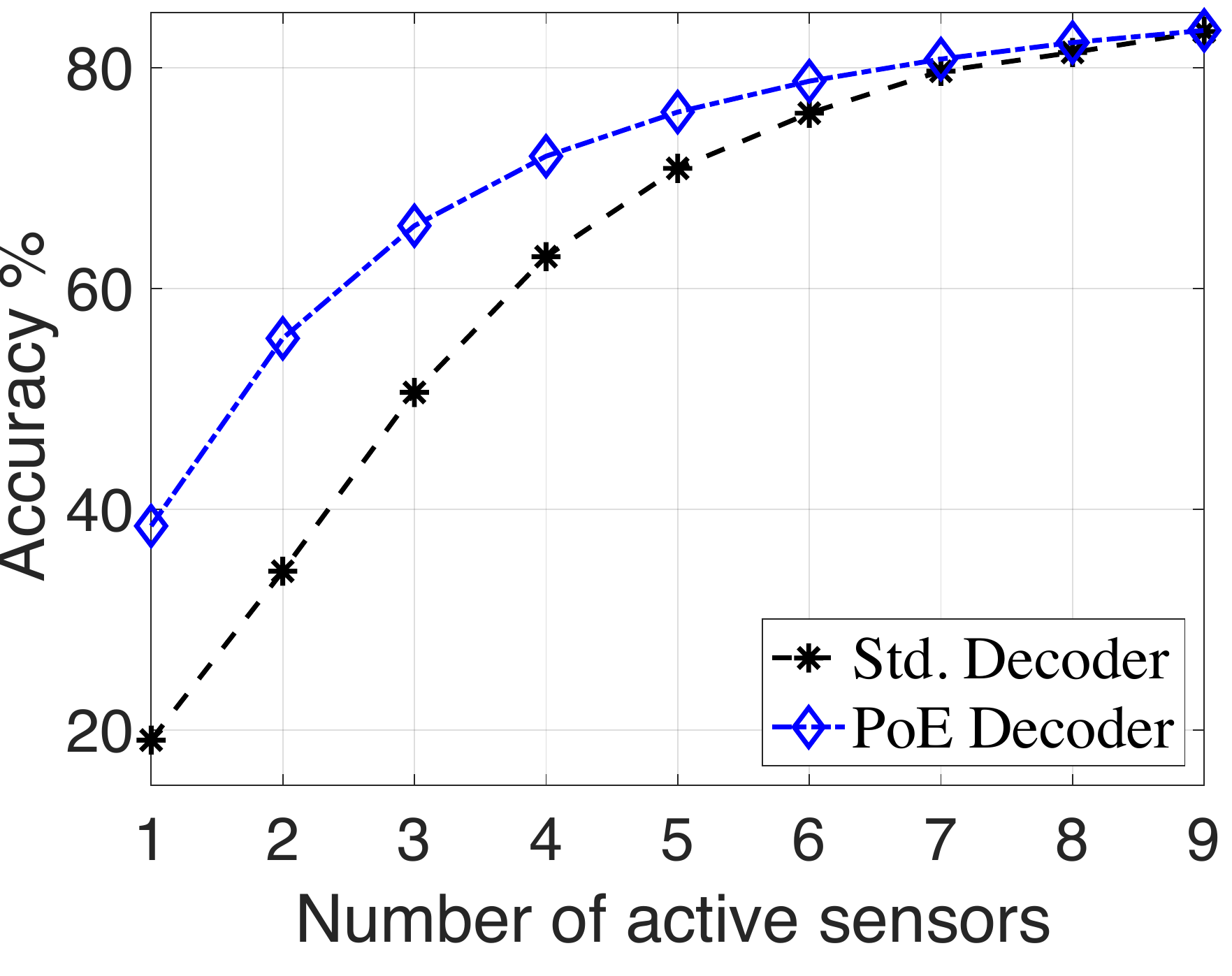}
			\caption{$N=9$ sensors.}
			\label{fig:awgn_test_dropout_N9}
		\end{subfigure}
	\end{center}
	\vspace{-7mm}
	\caption{Sensor outage during test time over AWGN channel.}
	\label{fig:awgn_test_dropout}
	\vspace{-7mm}	
\end{figure}

\Cref{fig:awgn_test_dropout} shows the performance of the system as a function of the number of active sensors. We initially train the system for $N$ sensors but assume that only a subset is active during testing. We compare our PoE decoder with the standard decoder. In the standard decoder, we concatenate the received transmissions of all sensors into a single vector, which forms the input to the decoder network. If a sensor drops out, then its corresponding part of the input is replaced with zeros. However, as the number of active sensors decreases, the input space of the decoder at the test time differs from the input space at train time. Thus, the performance degrades. In a PoE decoder, when a group of sensors $\mathcal{S}$ cannot send information to the edge router, the output is computed as $\prod_{n=1,n \notin \mathcal{S}}^{N}g_d(\vecthat{y}_n,n)$, i.e., we ignore sensors that did not transmit. Thus the PoE decoder performs better. Also, the performance of the standard decoder is very close to the PoE decoder when all nodes are active, indicating no loss of performance due to the PoE assumption.
\vspace{-5mm}

\subsection{Simulation results for Gaussian MAC}
\subsubsection{Varying upper bound on channel capacity}
Since the channel capacity of GMAC is unknown, following the work of \cite{lapidoth2010sending}, we upper bound it using the capacity of an AWGN channel as $(K/2)\log_2 \left(1+((N^2-N+1)P_T)/\sigma^2_z\right)$. \Cref{tab:gmac_varying_channel_cap} shows the performance of the systems trained using the three loss functions. The information bottleneck outperforms all other methods. The GMAC systems perform better than their AWGN counterparts, especially at lower capacities, probably because the superposition yields more protection from noise.

\begin{table}
	\vspace{-3mm}	
	\centering
	\caption{Classification performance for $N=4$ over GMAC for various channel capacity \emph{upper bounds}.}
	\vspace{-5mm}
	\begin{tabular}{|c||c|c|c|c|c|}
		\hline
		Method & C$\leq 12$ & C$\leq 16$ & C$\leq 20$  \\
		\hline
		\hline
		Autoencoder & 65.75\% & 71.72\%  & 74.00\% \\
		\hline
		Lagrange Method & 80.61\%  & 81.93\% & 82.51\% \\
		\hline
		Information Bottleneck & 81.54\%  & 83.19\%  & 84.00\% \\
		\hline
	\end{tabular}
	\label{tab:gmac_varying_channel_cap}
	\vspace{-7mm}
\end{table}

\Cref{fig:gmac_robustness} shows the robustness of the learned encoders and decoders by varying the channel conditions. $C_{tr}$ indicates the training capacity and $C_{te}$ the  capacity at test time. We see that even when $|C_{tr}-C_{te}|=8$ bits, the loss of performance is only around $0.5\%$ w.r.t. a system trained at $C_{te}$.
\vspace{-5mm}

\subsubsection{Varying number of sensors}
\begin{table}[h]
	\centering
	\caption{GMAC performance for varying $N$ with $C\leq 20$.}
	\begin{tabular}{|c||c|c|c|c|c|}
		\hline
		Method & $N=4$ & $N=9$ & $N=16$  \\
		\hline
		\hline
		IB+3S & 84.00\% & 83.52\% & 79.83\% \\
		\hline
		IB+3S+E2E & 83.92\% & 83.57\% & 79.21\% \\
		\hline
	\end{tabular}
	\label{tab:gmac_varying_sensors}
	\vspace{-7mm}
\end{table}
\Cref{tab:gmac_varying_sensors} shows the performance as the number of sensors increases. Notice that for $N=16$, the performance is lower than $N=4$. This is because the upper bound on the capacity of the GMAC channel becomes looser with increasing $N$, thus increasingly underestimating the required $P_T$. The learned solution has a capacity lower than $15$ bits which results in lower performance.

\begin{table}[h]
	\centering
	\caption{Number of channel uses for varying $N$ over GMAC with $C\leq 20$ bits.}
	\label{tab:gmac_channel_use_varying_sensors}
	\begin{tabular}{|c|c||c|c|c|}
		\hline
		\multicolumn{2}{|c||}{Method} & CT & IB+3S+E2E & IB+3S \\ 
		\hline
		\hline
		\multirow{2}{*}{$N=4$}  & S$\rightarrow$R  & $3.6e8$  & $2.9e7$ & $3.2e6$ \\ \cline{2-5} 
		& R$\rightarrow$S  & $3.6e7$   & $4.0e5$ & $7.2e6$ \\ 
		\hline \hline
		\multirow{2}{*}{$N=9$}  & S$\rightarrow$R  & $9.4e8$ & $8.6e7$ & $5.4e6$ \\ \cline{2-5} 
		& R$\rightarrow$S  & $8.1e7$  & $1.2e6$ & $1.1e7$ \\
		\hline \hline
		\multirow{2}{*}{$N=16$}  &  S$\rightarrow$R & $3.0e9$ & $>4.0e8$  & $2.1e7$ \\ \cline{2-5} 
		&  R$\rightarrow$S & $1.4e8$ & $>5.6e6$ & $3.4e7$ \\
		\hline
	\end{tabular}
	\vspace{-7mm}
\end{table}
In \cref{tab:gmac_channel_use_varying_sensors}, we show the channel uses for three methods. In the centralized scheme, since the transmission from the sensor to the router is digital, the nodes have to transmit in a time-sharing fashion. We assume that $K$ channel uses, on average, corresponds to $20$ bits of data transmitted. The values of $K=8,9,16$ for $N=4,9,16$, respectively. Thus each sensor node transmits $C/(KN)$ bits per channel use. We model the router to sensor channel as described in \cref{subsubsec:awgn_varying_number_sensors}.
Using the lossless JPEG2000 based system described in \cref{subsubsec:awgn_varying_number_sensors} for the Cloud Training (CT) scheme, we see that our proposed training schemes are more efficient in communication. The IB+3S (described in \cref{subsubsec:gmac_stage3}) is more efficient than IB+3S+E2E  for the uplink transmission by nine times to $>19$ times as $N$ increases. Similar to the AWGN case in \cref{subsubsec:awgn_varying_number_sensors}, IB+3S+E2E is more efficient in the downlink communication. However, the efficiency of IB+3S+E2E over IB+3S reduces from $18$ times to $<6$ times as $N$ increases. The increase in number of iterations as $N$ increases, allows IB+3S to catch up. Additionally, the number of communication rounds for IB+3S in the third stage is $T=6,11,25$ for $N=4,9,16$, respectively. This closely follows the relation predicted by \cref{remark:comm_rounds_noiseless}. The number of channel uses for everything except the downlink in IB+3S is computed using the mechanism described in \cref{subsubsec:awgn_varying_number_sensors}. We compute the downlink uses for IB+3S as $BT(|\mathcal{V}|+K)$.

\subsubsection{Sensor outages during testing}

\begin{figure}[htbp]
	\vspace{-3mm}	
	\begin{center}
		\begin{subfigure}[b]{0.49\textwidth}
			\centering
			\includegraphics[width=0.7\linewidth]{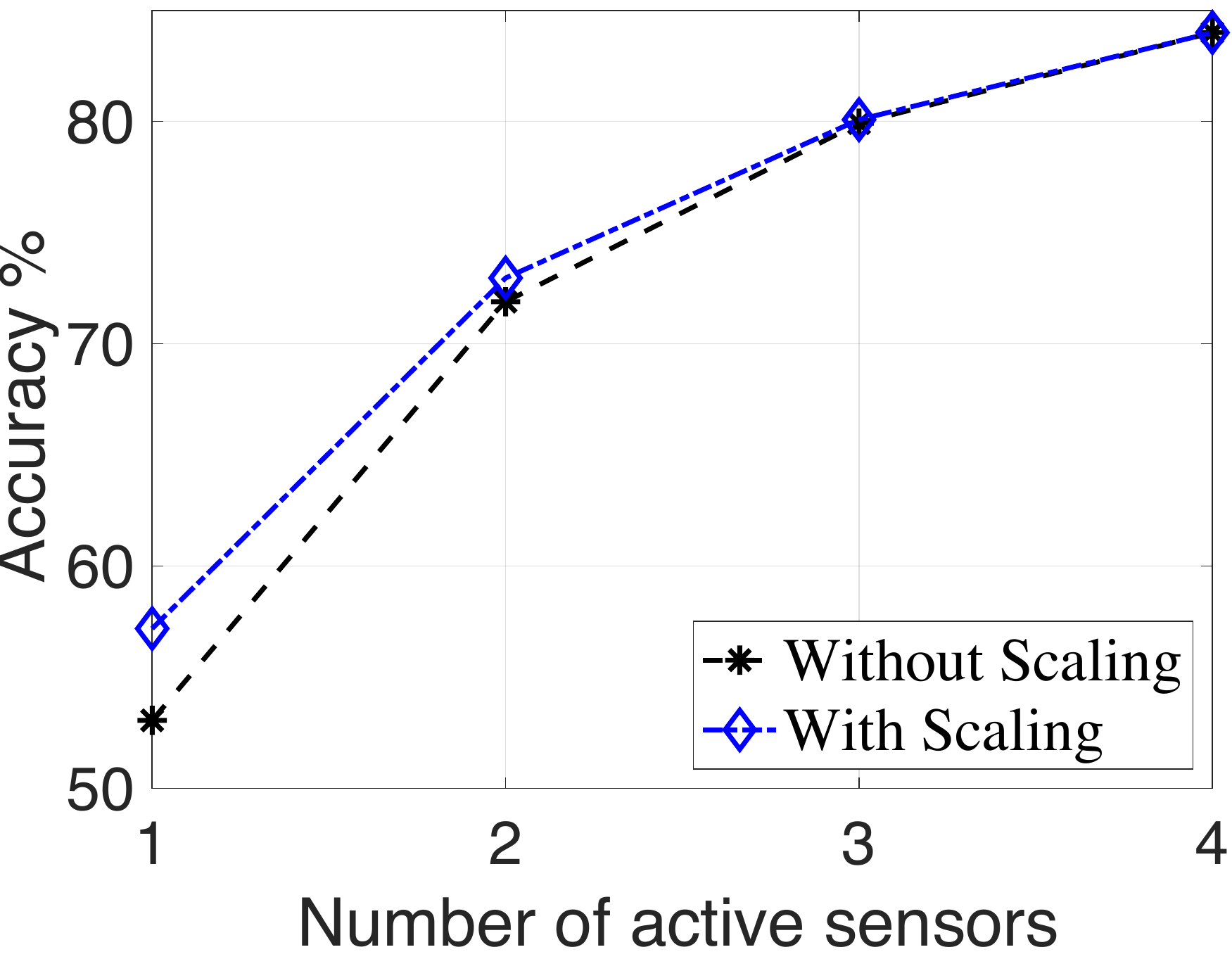}
			\caption{$N=4$ sensors.}
			\label{fig:gmac_test_dropout_N4}
		\end{subfigure} \hfill
		\begin{subfigure}[b]{0.49\textwidth}
			\centering
			\includegraphics[width=0.7\linewidth]{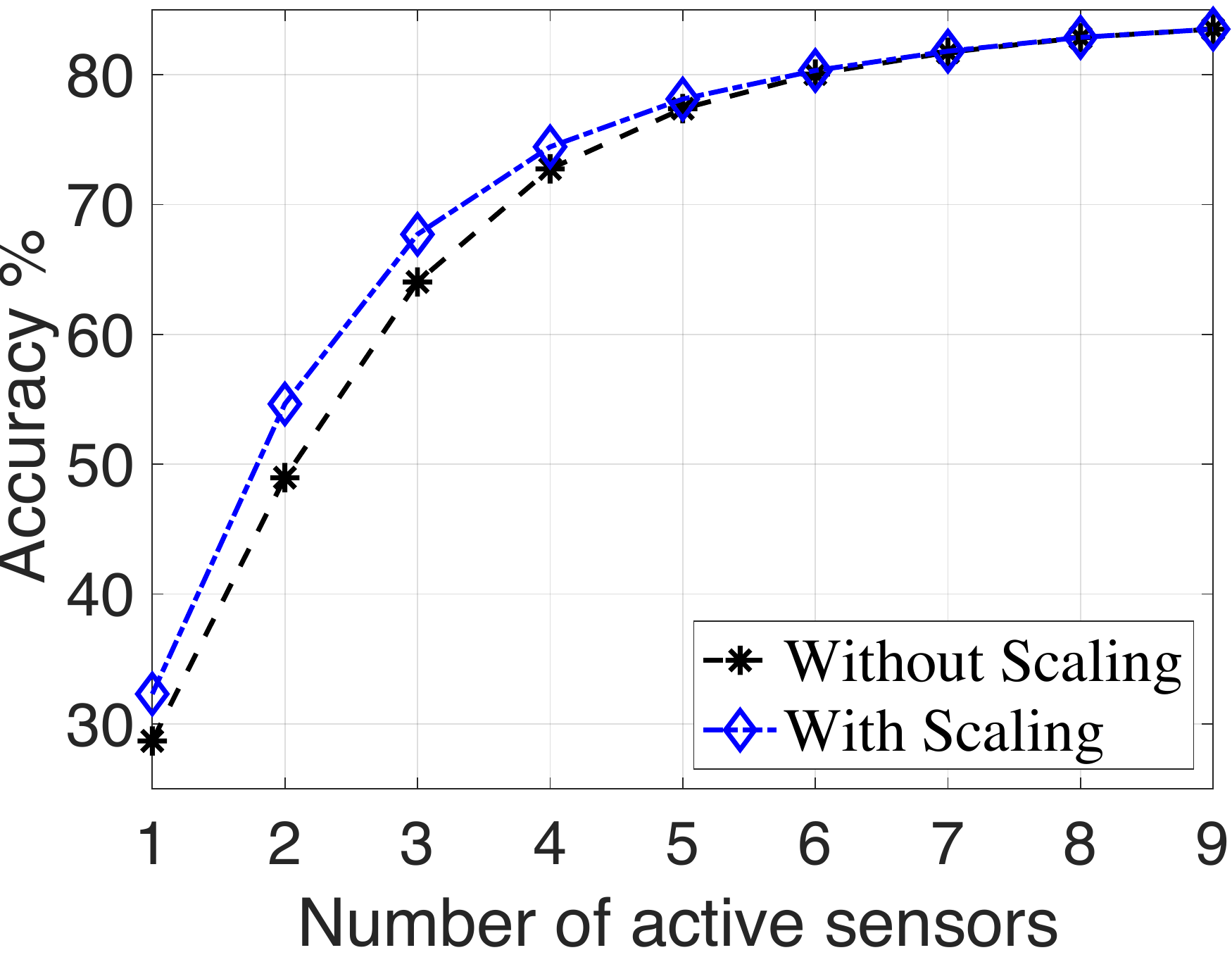}
			\caption{$N=9$ sensors.}
			\label{fig:gmac_test_dropout_N9}
		\end{subfigure}
	\end{center}
	\vspace{-7mm}	
	\caption{Sensor dropout during test time for $N=4$ and $N=9$ sensors over the GMAC.}
	\label{fig:gmac_test_dropout}
	\vspace{-8mm}	
\end{figure}

Unlike the orthogonal AWGN case, we cannot use a PoE decoder because the received $\vecthat{y}$ is the sum of all the transmitted $\{\vect{y}\}_{n=1}^N$. However, when only $|\mathcal{S}|$ nodes transmit, we can scale the received signal by $\frac{N}{|\mathcal{S}|}$. This scaled version performs better than the unscaled case, as seen in \Cref{fig:gmac_test_dropout}.

\section{Conclusion}
In this paper, we developed the first machine learning framework for distributed functional compression over wireless channels like GMAC and orthogonal AWGN in IoT settings. The sensor nodes observe the data in a distributed fashion and communicate without coordination to an edge router that approximates the function value. We looked at three different loss functions where the training is end-to-end. However, such training requires continuous communication between the sensor nodes and the edge router. Especially during the beginning, the encoder transmissions and the decoder feedback are not informative, and a lot of communication bandwidth is wasted. To overcome this, we proposed a three-stage training framework. The first two stages ensure that the encoder transmissions and the gradient feedback from the edge decoder are informative about the target function when the actual communication begins. When the target function is classification, we further formulated an improved training scheme that exploits the channel structure to remove the need for end-to-end training. For the orthogonal AWGN channel, we leveraged product-of-experts to design a decoder that is inherently robust to sensor outage. We provided convergence guarantees and a bound on the number of communication rounds for this training scheme. Our simulations showed that both the distributed training frameworks significantly reduce communication requirements compared to a cloud-based setup. Additionally, the proposed framework significantly outperforms traditional methods using Joint source-channel Coding. Finally, we showed that the learned encoders and decoders are robust to change in channel conditions and sensor outage.

\bibliographystyle{IEEEtran}
\bibliography{paper}

\clearpage
\newpage
\appendix

\subsection{Proof of \Cref{theo:IB_SL_AU_IRD_objective}}
\label{subsec:proof_theo_IB_SL_AU_IRD_objective}
\begin{proof}
	$\vect{X}^N \leftrightarrow \vect{Y}^N \leftrightarrow \vecthat{Y} \leftrightarrow \vecthat{V}$ is a Markov Chain. From data processing inequality, $I(\vect{X}^N;\vecthat{V}) \leq I(\vect{X}^N;\vecthat{Y}) \leq I(\vect{X}^N;\vect{Y}^N)$. Since channel noise is independent and additive, $I(\vect{X}^N;\vecthat{Y}) = H(\vecthat{Y}) - H(\vect{Z})$, where $H$ represents differential entropy. Since we do not know the distribution of $p(\vecthat{y})$, we use a variational approximation $r(\vecthat{y})$ to get \eqref{eqn:theo_IB}. As the encoders are deterministic, $I(\vect{X}^N;\vect{Y}^N)= H(\vect{Y}^N)$. We use variational approximations of the form $\prod_{n=1}r(\vect{y}_n)$. In the autoencoders, since all $\vect{y}_n$ have L2-norm $K_n P_T^{(n)}$, if we assume $r(\vecthat{y}_n)$ to be a uniform distribution on the surface of the $K_n$-dimensional hypersphere of radius $\sqrt{K_n P_T^{(n)}}$, we get \eqref{eqn:theo_AU}. In the Lagrange multiplier method, if we assume that $r_{L}(\vect{y}_n) = \mathcal{N}(\vect{y}_n;\vect{0},K_nP_T^{(n)}I)$, we get \eqref{eqn:theo_SL}. %Thus, they are variational approximations of $\expectation_{{V},\hat{V}} \left[\mathcal{D}({v},\hat{v})\right] + \lambda I(\vect{X}^N;\vect{Y}^N)$.
\end{proof}

\subsection{Proof of \Cref{theo-grad-avg-upperbound-noiseless}}
\label{subsec:proof-theo-grad-avg-upperbound-noiseless}
We use the following standard equalities/inequalities. For any $\vect{u}_1, \vect{u}_2, \dots, \vect{u}_M \in \mathbb{R}^P$, and $0<\mu^2<1$:
\begin{subequations}
	\begin{equation}
		2\innerprod{\vect{u}_i}{\vect{u}_j} \leq \mu^2 \norm{\vect{u}_i}^2_2 + \frac{1}{\mu^2} \norm{\vect{u}_j}^2_2
		\label{eqn:ip_sum_ineq}
	\end{equation}
	\begin{equation}
		\norm{\sum_{i=1}^{M} \vect{u}_i}_2^2 = \sum_{i=1}^{M} \norm{\vect{u}_i}^2_2 \text{ if } \innerprod{\vect{u}_i}{\vect{u}_j} = 0 \forall i \neq j
		\label{eqn:normsq_of_sumeq}
	\end{equation}
\end{subequations}

Let $\maxone{x} \defeq \max \left(1,x\right)$.

Let $\paramstep{\vect{\Delta}}{s} \defeq \paramstep{\vect{\Theta}}{s+1}-\paramstep{\vect{\Theta}}{s}$. By \cref{eqn:normsq_of_sumeq} and using triangle inequality we can show that, $\norm{\paramstep{\vect{\Theta}}{s}-\paramstep{\vecthat{\Theta}}{r(s,n),n}}_2 \leq \sum_{i=\maxone{s-2\tau}}^{s-1} \norm{\paramstep{\vect{\Delta}}{i}}_2$. Since $ab \leq \frac{1}{2}(a^2+b^2)$,
\begin{equation}
	\norm{\paramstep{\vect{\Theta}}{s}-\paramstep{\vecthat{\Theta}}{r(s,n),n}}^2_2 \leq \left(\tau + \frac{1}{2}\right)\sum_{i=\maxone{s-2\tau}}^{s-1} \norm{\paramstep{\vect{\Delta}}{i}}^2_2.
	\label{eqn:corr-diff-delta-sq-twotau}
\end{equation} 

Define a Lyapunov function of the form
\begin{equation}
	\zeta_s = \mathcal{L}\left(\paramstep{\vect{\Theta}}{s}\right) - \mathcal{L^*} + \frac{\mathrm{L}}{2\epsilon} \sum_{i=\maxone{s-\tau}}^{s-1} (i-(s-\tau)+1) \norm{\paramstep{\Delta}{i}}_2^2,
	\label{eqn:lyapunov_def}
\end{equation}
 where $\epsilon>0$ satisfies $\epsilon + \frac{1}{\epsilon} = 1 + \frac{1}{\tau}\left(\frac{1}{\eta}-\frac{1}{2}\right)$.
%Note that $\zeta_s \geq 0$.

\begin{lemma}
	If assumptions 1 and 2 hold, then
	\begin{equation}
		\zeta_s - \zeta_{s+1} \geq \frac{\mathrm{L}\left(\alpha-1\right)}{2} \left(\tau + \frac{1}{2}\right) \norm{\paramstep{\vect{\Delta}}{s}}_2^2 \geq 0.
		\label{eqn:lemma-lyapunovdiff-delta}
	\end{equation}
	\label{lemma-lyapunovdiff-delta}
	\vspace{-12mm}
\end{lemma}
\begin{proof}
	Proof follows similar to \cite[Lemma 1]{sun2017asynchronous}.
\end{proof}

\begin{proof} \emph{for \cref{theo-grad-avg-upperbound-noiseless}}.
	From \eqref{eqn:normsq_of_sumeq}, $\norm{\nabla \Lloss \left(\paramstep{\vect{\Theta}}{s}\right)}_2^2 = \sum_n \norm{\nabla_n \Lloss \left(\paramstep{\vect{\Theta}}{s}\right)}_2^2$.
	From assumption 1 and  \eqref{eqn:ip_sum_ineq}, we get
	\begin{equation}
		\norm{\nabla_n \Lloss \left(\paramstep{\vect{\Theta}}{s}\right)}_2^2 \leq \frac{1}{\mu^2}\norm{\nabla_n \Lloss \left(\paramstep{\vecthat{\Theta}}{r(s,n),n}\right)}_2^2 + \frac{\mathrm{L}^2}{1-\mu^2}	\norm{\paramstep{\vect{\Theta}}{s}- \paramstep{\vecthat{\Theta}}{r(s,n),n}}_2^2,
		\label{eqn:theo1_step1}
	\end{equation}
	where $\mu^2 \in \{0,1\}$. Using \cref{lemma-lyapunovdiff-delta} to bound the first term, and \eqref{eqn:corr-diff-delta-sq-twotau} and \cref{lemma-lyapunovdiff-delta} on the second term, we get
	\begin{multline}
		\norm{\nabla_n \Lloss \left(\paramstep{\vect{\Theta}}{s}\right)}_2^2 \leq \frac{1}{\mu^2} \left(\frac{2\mathrm{L} \alpha^2}{ \left(\alpha-1\right)}\left(\tau+\frac{1}{2}\right)\right)\left(\zeta_{r(s,n)} - \zeta_{r(s,n)+1}\right) \\
		+ \frac{4\mathrm{L}\tau}{\left(1-\mu^2\right)\left(\alpha-1\right)\left(2\tau+1\right)}	\sum_{i=\maxone{s-2\tau}}^{s-1} \zeta_i -\zeta_{i+1}.
		\label{eqn:theo1_bound_on_L2}
	\end{multline}
	Let us define $A_1$ as the summation of the first term of \eqref{eqn:theo1_bound_on_L2} over $\sum_{n=1}^{N+1}\sum_{s=1}^{S}$. Similarly, define $A_2$ for the second term.
		
	To bound $A_1$, note that the term $\zeta_{r(s,n)} - \zeta_{r(s,n)+1} \geq 0$ can repeat at most $\tau$ times. Thus $\sum_{s=1}^{S} \left(\zeta_{r(s,n)} - \zeta_{r(s,n)+1}\right) < \tau \sum_{s=1:r(s,n)\neq r(s-1,n)}^{S} \left(\zeta_{r(s,n)} - \zeta_{r(s,n)+1}\right)$. The double summation after reintroducing $\sum_{n=1}^{N+1}$ reduces to a single summation of the form $\sum_{s=1}^{S} \zeta_s - \zeta_{s+1}$, because one parameter block gets updated in an iteration step.
	We can split $\sum_{s=1}^{S} \sum_{i=\maxone{s-2\tau}}^{s-1} \zeta_i -\zeta_{i+1}$ into two summations over $s=1$ to $2\tau$ and other $s=2\tau+1$ to $S$.
	We can show that the first summation is $\leq (2\tau-1)\zeta_1$ and the second summation is $\leq 2\tau \zeta_1$. Thus $A_2$ is bounded.
	
	Using above and replacing $\tau$ by $(2N+1)E$ we get \eqref{eqn:theo-grad-avg-upperbound-noiseless}. Further, since $\norm{\nabla 	\mathcal{L}\left(\paramstep{\vect{\Theta}}{s}\right)}_2^2$ is a summable non-negative sequence, it converges to $0$.
\end{proof}

\subsection{Proof of \Cref{theo-grad-avg-upperbound-noisy-v2}.}
\label{sec:proof-theo-grad-avg-upperbound-noisy}
Note that $\vect{E}_n$ represents the random variable representing the noise in the gradient approximation at node-$n$ and $E$ represents the number of local iterations. 

\begin{lemma}
	Define $\paramstep{\vect{\Delta}}{k} \defeq \paramstep{\vect{\Theta}}{k+1}-\paramstep{\vect{\Theta}}{k}$. Then,
	\begin{equation}
		\norm{\paramstep{\vect{\Theta}}{k}-\paramstep{\vecthat{\Theta}}{k,l}}_2 \leq \sum_{i=\maxone{k-\tau}}^{k-1} \norm{\paramstep{\vect{\Delta}}{i}}_2
		\label{eqn:lemma1}
	\end{equation} 
	\label{lemma-diff-delta}
	\vspace{-12mm}
\end{lemma}
\begin{proof}
	\begin{equation}
		\norm{\paramstep{\vect{\Theta}}{k}-\paramstep{\vecthat{\Theta}}{k,l}}_2 \leq \sum_{m=1}^{N+1} \norm{\paramstep{\vect{\Theta}}{k}_m-\paramstep{\vecthat{\Theta}}{k,l}_m}_2
		\label{eqn:lemma1_step1}
	\end{equation}
	
	\textbf{Case (1): $m \neq l$ -} Let $k' = \floor{\frac{k}{(N+1)E}}(N+1)E$. It follows that	\begin{equation}
		\norm{\paramstep{\vect{\Theta}}{k}_m-\paramstep{\vecthat{\Theta}}{k,l}_m}_2 =  	\norm{\paramstep{\vect{\Theta}}{k}_m-\paramstep{\vect{\Theta}}{r(k',m)}_m}_2 = \norm{\sum_{i=r(k',m)}^{k-1} \paramstep{\vect{\Theta}}{i+1}_m-\paramstep{\vect{\Theta}}{i}_m}_2 = 
		\norm{\sum_{i=r(k',m): i+1=r(i+1,m)}^{k-1} \paramstep{\vect{\Delta}}{i}}_2.
	\end{equation}
	The last step follows because we only need to sum over those $i$ when $\vect{\Theta}_m$ is updated. Thus it follows that,
	\begin{equation}
		\norm{\paramstep{\vect{\Theta}}{k}_m-\paramstep{\vecthat{\Theta}}{k,l}_m}_2 \leq \sum_{i=r(k',m): i+1=r(i+1,m)}^{k-1}  \norm{\paramstep{\vect{\Delta}}{i}}_2 \leq \sum_{i=\maxone{k-\tau}: i+1=r(i+1,m)}^{k-1}  \norm{\paramstep{\vect{\Delta}}{i}}_2
		\label{eqn:lemma1_step2_mneql}
	\end{equation}
	Here, the first inequality follows from triangle inequality. The second inequality follows because $\maxone{k-\tau} \leq r(k',m)$ and $\norm{\cdot}_2 \geq 0$. 
	
	\textbf{Case (2): $m=l$ -} It follows that
	\begin{equation}
		\norm{\paramstep{\vect{\Theta}}{k}_m-\paramstep{\vecthat{\Theta}}{k,l}_m}_2 = 
		\norm{\paramstep{\vect{\Delta}}{r(k,l)}}_2 \leq \sum_{i=\maxone{k-\tau}: i+1=r(i+1,m)}^{k-1}  \norm{\paramstep{\vect{\Delta}}{i}}_2
		\label{eqn:lemma1_step2_meql}
	\end{equation}

	Thus using \eqref{eqn:lemma1_step2_mneql} and \eqref{eqn:lemma1_step2_meql} in \eqref{eqn:lemma1_step1}, we have 
	\begin{equation}
		\norm{\paramstep{\vect{\Theta}}{k}-\paramstep{\vecthat{\Theta}}{k,l}}_2 \leq \sum_{m=1}^{N+1} \left( \sum_{i=\maxone{k-\tau}: i+1=r(i+1,m)}^{k-1}  \norm{\paramstep{\vect{\Delta}}{i}}_2 \right)
		\label{eqn:lemma1_step3}
	\end{equation}
	Since every step in the parameter update only updates one block of parameters, it follows that RHS of \eqref{eqn:lemma1_step3} is the same as the RHS of \eqref{eqn:lemma1}.
\end{proof}

%From \Cref{lemma-diff-delta} and the the fact that $(a-b)^2 \geq 0 \implies ab \leq \frac{1}{2}\left( a^2+b^2 \right)$ we can show that
%	\begin{equation}
%	\norm{\paramstep{\vect{\Theta}}{k}-\paramstep{\vecthat{\Theta}}{k,l}}^2_2 \leq \frac{\tau+1}{2}\sum_{i=\maxone{k-\tau}}^{k-1} \norm{\paramstep{\vect{\Delta}}{i}}^2_2.
%	\label{corr-diff-delta-sq}
%\end{equation} 

\begin{corr}
	\begin{equation}
		\norm{\paramstep{\vect{\Theta}}{k}-\paramstep{\vecthat{\Theta}}{r(k,l),l}}_2 \leq 	\sum_{i=\maxone{k-2\tau}}^{k-1} \norm{\paramstep{\vect{\Delta}}{i}}_2
		\label{eqn:corr2}
	\end{equation} 
	\label{corr-diff-delta-twotau}
	\vspace{-12mm}
\end{corr}
\begin{proof}
	Denote $k''=r(k,l)$. W.k.t. $\maxone{k''-\tau+1}\leq r(k'',l) \leq k''$ and $\maxone{k-\tau+1}\leq r(k,l) \leq k$. Thus, $\maxone{k-2\tau+1} \leq r(r(k,l),l)\leq k$. Following the same methodology as the proof of \cref{lemma-diff-delta} and changing the limits of the summations in \eqref{eqn:lemma1_step2_mneql} and \eqref{eqn:lemma1_step2_meql}, we get the result.
\end{proof}

%\begin{corr}
%	\begin{equation}
%		\norm{\paramstep{\vect{\Theta}}{k}-\paramstep{\vecthat{\Theta}}{k,l}}^2_2 \leq \frac{\tau+1}{2}\sum_{i=\maxone{k-\tau}}^{k-1} \norm{\paramstep{\vect{\Delta}}{i}}^2_2
%		\label{eqn:corr-diff-delta-sq}
%	\end{equation} 
%	\label{corr-diff-delta-sq}
%\end{corr}
%\begin{proof}
%	From \cref{lemma-diff-delta} we have
%	\begin{equation}
%		\norm{\paramstep{\vect{\Theta}}{k}-\paramstep{\vecthat{\Theta}}{k,l}}^2_2 \leq \sum_{i=\maxone{k-\tau}}^{k-1}\sum_{j=\maxone{k-\tau}}^{k-1} \norm{\paramstep{\vect{\Delta}}{i}}_2 \norm{\paramstep{\vect{\Delta}}{j}}_2.
%	\end{equation}
%	W.k.t. $(a-b)^2 \geq 0 \implies ab \leq \frac{1}{2}\left( a^2+b^2 \right)$. Thus, we have
%	\begin{equation}
%		\norm{\paramstep{\vect{\Theta}}{k}-\paramstep{\vecthat{\Theta}}{k,l}}^2_2 \leq \frac{1}{2}\sum_{i=\maxone{k-\tau}}^{k-1}\sum_{j=\maxone{k-\tau}}^{k-1} \left( \norm{\paramstep{\vect{\Delta}}{i}}^2_2 +  \norm{\paramstep{\vect{\Delta}}{j}}^2_2 \right) \leq \frac{\tau+1}{2}\sum_{i=\maxone{k-\tau}}^{k-1} \norm{\paramstep{\vect{\Delta}}{i}}^2_2.
%	\end{equation}
%\end{proof}
%
%\begin{corr}
%	\begin{equation}
%		\norm{\paramstep{\vect{\Theta}}{k}-\paramstep{\vecthat{\Theta}}{r(k,l),l}}^2_2 \leq \left(\tau + \frac{1}{2}\right)\sum_{i=\maxone{k-2\tau}}^{k-1} \norm{\paramstep{\vect{\Delta}}{i}}^2_2
%		\label{eqn:corr-diff-delta-sq-twotau}
%	\end{equation} 
%	\label{corr-diff-delta-sq-twotau}
%\end{corr}
%\begin{proof}
%	Use \cref{corr-diff-delta-twotau} and the steps outlined in the proof of \cref{corr-diff-delta-sq}.
%\end{proof}

From \Cref{lemma-diff-delta} and the the fact that $(a-b)^2 \geq 0 \implies ab \leq \frac{1}{2}\left( a^2+b^2 \right)$ we can show that
\begin{equation}
	\norm{\paramstep{\vect{\Theta}}{k}-\paramstep{\vecthat{\Theta}}{k,l}}^2_2 \leq \frac{\tau+1}{2}\sum_{i=\maxone{k-\tau}}^{k-1} \norm{\paramstep{\vect{\Delta}}{i}}^2_2.
	\label{corr-diff-delta-sq}
\end{equation} 
By following the same steps but starting from \Cref{corr-diff-delta-twotau} we can show that
\begin{equation}
	\norm{\paramstep{\vect{\Theta}}{k}-\paramstep{\vecthat{\Theta}}{r(k,l),l}}^2_2 \leq \left(\tau + \frac{1}{2}\right)\sum_{i=\maxone{k-2\tau}}^{k-1} \norm{\paramstep{\vect{\Delta}}{i}}^2_2
	\label{corr-diff-delta-sq-twotau}
\end{equation} 

%\begin{corr}
%	\begin{equation}
%		\norm{\paramstep{\vect{\Theta}}{k}-\paramstep{\vecthat{\Theta}}{k,l}}^2_2 \leq \frac{\tau+1}{2}\sum_{i=\maxone{k-\tau}}^{k-1} \norm{\paramstep{\vect{\Delta}}{i}}^2_2
%		\label{eqn:corr-diff-delta-sq}
%	\end{equation} 
%	\label{corr-diff-delta-sq}
%\end{corr}
%\begin{proof}
%	Follows from \Cref{lemma-diff-delta} and the the fact that $(a-b)^2 \geq 0 \implies ab \leq \frac{1}{2}\left( a^2+b^2 \right)$.
%	From \cref{lemma-diff-delta} we have
%	\begin{equation}
%		\norm{\paramstep{\vect{\Theta}}{k}-\paramstep{\vecthat{\Theta}}{k,l}}^2_2 \leq \sum_{i=\maxone{k-\tau}}^{k-1}\sum_{j=\maxone{k-\tau}}^{k-1} \norm{\paramstep{\vect{\Delta}}{i}}_2 \norm{\paramstep{\vect{\Delta}}{j}}_2.
%	\end{equation}
%	W.k.t. $(a-b)^2 \geq 0 \implies ab \leq \frac{1}{2}\left( a^2+b^2 \right)$. Thus, we have
%	\begin{equation}
%		\norm{\paramstep{\vect{\Theta}}{k}-\paramstep{\vecthat{\Theta}}{k,l}}^2_2 \leq \frac{1}{2}\sum_{i=\maxone{k-\tau}}^{k-1}\sum_{j=\maxone{k-\tau}}^{k-1} \left( \norm{\paramstep{\vect{\Delta}}{i}}^2_2 +  \norm{\paramstep{\vect{\Delta}}{j}}^2_2 \right) \leq \frac{\tau+1}{2}\sum_{i=\maxone{k-\tau}}^{k-1} \norm{\paramstep{\vect{\Delta}}{i}}^2_2.
%	\end{equation}
%\end{proof}

\begin{definition}
	A Lyapunov function $\xi_s$ is defined as
	\begin{equation}
		\xi_s = \Lloss\left(\paramstep{\vect{\Theta}}{s}\right) - \Lloss^* + \frac{\eta \rho_0 \mathrm{L} (\tau+1)}{4} \sum_{i=\maxone{s-\tau}}^{s-1} (i-(s-\tau)+1) \norm{\paramstep{\vect{\Delta}}{i}}_2^2,
	\end{equation}
	where $\rho_0 > 0$ will be determined later.
\end{definition}

\begin{lemma}
	If Assumptions 1, 3, 4, and 5 are satisfied then,
	\begin{equation}
		\frac{\mathrm{L}}{\eta(1-\eta(\tau+1))}\expectation \left[ \xi_s - \xi_{s+1} \right] + \frac{\eta \left( \tau + 2 \right)}{4(1-\eta(\tau+1))} \sigma^2_n  \geq \expectation \left[ \norm{\nabla_{n} \mathcal{L}\left( \paramstep{\hatvect{\Theta}}{s,n} \right)}_2^2 \right] \geq 0.
	\end{equation}
	\label{lemma-lyapunov-diff-grad-square-noisy-special}
	\vspace{-10mm}
\end{lemma}
\begin{proof} 
	\begin{equation}
		\xi_s - \xi_{s+1} \\
		= \Lloss\left(\paramstep{\vect{\Theta}}{s}\right)  - \Lloss\left(\paramstep{\vect{\Theta}}{s+1}\right) + \frac{\eta \rho_0 \mathrm{L} (\tau+1)}{4} \sum_{i=\maxone{s-\tau}}^{s-1} \norm{\paramstep{\vect{\Delta}}{i}}_2^2 - \frac{\eta \rho_0 \mathrm{L} \tau (\tau+1)}{4} \norm{\paramstep{\Delta}{s}}_2^2.
		\label{eqn:lemma-lyapunov-diff-grad-square-noisy-special-step1}
	\end{equation}
	By Assumption 1 and taking expectation $\expectation_{\vect{E}_n^{(s)}}$ on both sides, we get
	\begin{equation}
		\innerprod{\nabla_{\vect{\Theta}}\Lloss\left( \paramstep{\vect{\Theta}}{s} \right)}{\expectation_{\vect{E}_n^{(s)}} \left[ \paramstep{\Delta}{s} \right]} = \innerprod{\nabla_{\vect{\Theta}}\Lloss\left( \paramstep{\vect{\Theta}}{s} \right)}{-\frac{\eta}{\mathrm{L}}\nabla_{n} \mathcal{L}\left( \paramstep{\hatvect{\Theta}}{s,n} \right)}.
	\end{equation}
%	This follows because $\nabla_{\vect{\Theta}}\Lloss\left( \paramstep{\vect{\Theta}}{s} \right)$ does not depend on noise at the $s^{\text{th}}$ step. 
	By applying \eqref{eqn:ip_sum_ineq} (using $\rho_0 > 0$ instead of $\mu^2$) and Assumption 1, we get
	\begin{equation}
		\innerprod{\nabla_{\vect{\Theta}}\Lloss\left( \paramstep{\vect{\Theta}}{s} \right)}{\expectation_{\vect{E}_n^{(s)}} \left[ \paramstep{\Delta}{s} \right]} \leq \frac{\eta \mathrm{L} \rho_0}{2} \norm{\paramstep{\vect{\Theta}}{s} - \paramstep{\hatvect{\Theta}}{s,n}}_2^2 + \frac{\eta}{\mathrm{L}}\left(\frac{1}{2\rho_0}-1\right) \norm{\nabla_{n} \mathcal{L}\left( \paramstep{\hatvect{\Theta}}{s,n} \right)}_2^2.
	\end{equation}
	Let $\expectation$ represent the expectation w.r.t. all previous steps from $1,\dots,s$. Then,
	\begin{multline}
		\expectation \left[ \Lloss\left(\paramstep{\vect{\Theta}}{s}\right)  - \Lloss\left(\paramstep{\vect{\Theta}}{s+1}\right) \right] \geq - \frac{\eta \mathrm{L} \rho_0}{2} \expectation \left[ \norm{\paramstep{\vect{\Theta}}{s} - \paramstep{\hatvect{\Theta}}{s,n}}_2^2 \right] \\
		+ \frac{\eta}{\mathrm{L}}\left(1-\frac{1}{2\rho_0}\right) \expectation \left[ \norm{\nabla_{n} \mathcal{L}\left( \paramstep{\hatvect{\Theta}}{s,n} \right)}_2^2 \right] - \frac{\mathrm{L}}{2} \expectation \left[ \norm{\paramstep{\Delta}{s}}_2^2 \right].
	\end{multline}
	From, \Cref{corr-diff-delta-sq} and \eqref{eqn:lemma-lyapunov-diff-grad-square-noisy-special-step1} we get,
	\begin{equation}
		\expectation \left[ \xi_s - \xi_{s+1} \right] \geq 	\frac{\eta}{\mathrm{L}}\left(1-\frac{1}{2\rho_0}\right) \expectation \left[ \norm{\nabla_{n} \mathcal{L}\left( \paramstep{\hatvect{\Theta}}{s,n} \right)}_2^2 \right] \\
		- \frac{\mathrm{L}}{2} \expectation \left[ \norm{\paramstep{\Delta}{s}}_2^2 \right] - \frac{\eta \rho_0 \mathrm{L} \tau (\tau+1)}{4} \expectation \left[  \norm{\paramstep{\Delta}{s}}_2^2 \right].
	\end{equation}
	Since, $\paramstep{\Delta}{s} = -\frac{\eta}{\mathrm{L}} \left(\nabla_n \Lloss \left(\paramstep{\vecthat{\Theta}}{s,n}\right) + \vect{\epsilon}_n^{(s)}\right)$, we get
%	\begin{equation}
%		\expectation \left[ \xi_s - \xi_{s+1} \right] + 	\frac{\eta^2}{2\mathrm{L}} \left( 1+ \frac{\eta \rho_0 \tau (\tau+1)}{2} \right)\sigma^2_n  \\
%		\geq \frac{\eta}{\mathrm{L}}\left(1-\frac{1}{2\rho_0}  - \frac{\eta}{2} - \frac{\eta^2 \rho_0 \tau (\tau+1)}{2}\right) \expectation \left[ \norm{\nabla_{n} \mathcal{L}\left( \paramstep{\hatvect{\Theta}}{s,n} \right)}_2^2 \right].
%	\end{equation}
%	By rearranging terms we get,
	By using definition of $\paramstep{\Delta}{s}$, setting $\rho_0 = \frac{1}{\eta(\tau+1)} \geq 1$, and using Assumption 5, we get
%	\begin{equation}
%		\frac{2\mathrm{L}}{\eta \left(2-\frac{1}{\rho_0}  - \eta - \eta^2 	\rho_0 \tau (\tau+1) \right)}\expectation \left[ \xi_s - \xi_{s+1} \right] \\
%		+ \frac{\eta \left(2+ \eta \rho_0 \tau (\tau+1)\right)}{2 \left(2-\frac{1}{\rho_0}  - \eta - \eta^2 \rho_0 \tau (\tau+1)\right)} \sigma^2_n  \\
%		\geq \expectation \left[ \norm{\nabla_{n} \mathcal{L}\left( \paramstep{\hatvect{\Theta}}{s,n} \right)}_2^2 \right].
%	\end{equation}
%	We set $\rho_0 = \frac{1}{\eta(\tau+1)} \geq 1$, the inequality following from Assumption 5. This gives us
	\begin{equation}
		\frac{\mathrm{L}}{\eta(1-\eta(\tau+1))}\expectation \left[ \xi_s - 	\xi_{s+1} \right] + \frac{\eta \left( \tau + 2 \right)}{4(1-\eta(\tau+1))} \sigma^2_n  \geq \expectation \left[ \norm{\nabla_{n} \mathcal{L}\left( \paramstep{\hatvect{\Theta}}{s,n} \right)}_2^2 \right] \geq 0.
	\end{equation}
\end{proof}

\begin{corr}
	If Assumptions 1, 3, and 5 hold then,
	\begin{equation}
		\frac{4 \eta}{\mathrm{L} (4 - \eta (3\tau+4))} \expectation \left[ \xi_s - \xi_{s+1} \right] + \frac{4 \eta^2}{\mathrm{L}^2 (4 - \eta (3\tau+4))} \sigma^2_n \geq \expectation \left[ \norm{\paramstep{\Delta}{s}}_2^2 \right] \geq 0.
	\end{equation}	
	\label{corr-lyapunov-diff-delta-square-noisy-special}
	\vspace{-12mm}
\end{corr}
\begin{proof}
	From proof of \Cref{lemma-lyapunov-diff-grad-square-noisy-special} w.k.t.,
	\begin{multline}
		\expectation \left[ \xi_s - \xi_{s+1} \right] \geq 	\frac{\eta}{\mathrm{L}}\left(1-\frac{1}{2\rho_0}\right) \expectation \left[ \norm{\nabla_{n} \mathcal{L}\left( \paramstep{\hatvect{\Theta}}{s,n} \right)}_2^2 \right] \\
		- \frac{\mathrm{L}}{2} \expectation \left[ \norm{\paramstep{\Delta}{s}}_2^2 \right] - \frac{\eta \rho_0 \mathrm{L} \tau (\tau+1)}{4} \expectation \left[  \norm{\paramstep{\Delta}{s}}_2^2 \right].
	\end{multline}
%	W.k.t. $\paramstep{\Delta}{s} = -\frac{\eta}{\mathrm{L}} \left(\nabla_n \Lloss \left(\paramstep{\vecthat{\Theta}}{s,n}\right) + \vect{\epsilon}_n^{(s)}\right)$. Thus we get $\expectation \left[ \norm{\nabla_n \Lloss \left(\paramstep{\vecthat{\Theta}}{s,n} \right)}_2^2 \right]= \frac{\mathrm{L}^2}{\eta^2}$ $\expectation \left[ \norm{\paramstep{\Delta}{s}}_2^2 \right] -\sigma^2_n$. Thus,
	By using definition of $\paramstep{\Delta}{s}$, we get
	\begin{equation}
		\expectation \left[ \xi_s - \xi_{s+1} \right] \geq 	\left(\frac{\mathrm{L}}{\eta}\left(1-\frac{1}{2\rho_0}\right) - \frac{\mathrm{L}}{2} - \frac{\eta \rho_0 L \tau (\tau+1)}{4} \right) \expectation \left[ \norm{\paramstep{\Delta}{s}}_2^2 \right] - \frac{\eta}{\mathrm{L}}\left(1-\frac{1}{2\rho_0}\right) \sigma^2_n.
	\end{equation}	
%	This can be simplified as
%	\begin{equation}
%		\expectation \left[ \xi_s - \xi_{s+1} \right] \geq 	\frac{\mathrm{L}}{\eta} \left(1 - \eta \frac{3\tau+4}{4} \right) \expectation \left[ \norm{\paramstep{\Delta}{s}}_2^2 \right] - \frac{\eta}{\mathrm{L}}\left(1-\eta (\tau+1)\right) \sigma^2_n.
%	\end{equation}
	Note that $1 \geq \eta \frac{3\tau+4}{4}$ by Assumption 5. Thus,
	\begin{equation}
		\frac{4 \eta}{\mathrm{L} (4 - \eta (3\tau+4))} \expectation \left[ \xi_s - \xi_{s+1} \right] + \frac{4 \eta^2}{\mathrm{L}^2 (4 - \eta (3\tau+4))} \sigma^2_n \geq \expectation \left[ \norm{\paramstep{\Delta}{s}}_2^2 \right] \geq 0.
	\end{equation}
\end{proof}

Now we can the begin proof of \Cref{theo-grad-avg-upperbound-noisy-v2}.
\begin{proof}
	From \eqref{eqn:normsq_of_sumeq} we have
	\begin{equation}
		\norm{\nabla \mathcal{L}\left(\paramstep{\vect{\Theta}}{s}\right)}_2^2 = \sum_{l=1}^{N+1} \norm{\nabla_{n} \mathcal{L}\left(\paramstep{\vect{\Theta}}{s}\right)}_2^2.
		\label{eqn:theo-grad-avg-upperbound-noisy-v2-grad-norm-bound}
	\end{equation}
	This is because $\innerprod{\nabla_{n} \mathcal{L}\left(\paramstep{\vect{\Theta}}{s}\right)}{\nabla_{m} \mathcal{L}\left(\paramstep{\vect{\Theta}}{s}\right)}$ have non-zero values in mutually exclusive indices for $l \neq m$.

	We shift our focus to bounding $\norm{\nabla_{n} \mathcal{L}\left(\paramstep{\vect{\Theta}}{s}\right)}_2^2$. Based on Assumption 1, we can simplify above as
	\begin{multline}
		\norm{\nabla_n \Lloss \left(\paramstep{\vect{\Theta}}{s}\right)}_2^2 \leq -\norm{\nabla_n \Lloss \left(\paramstep{\vecthat{\Theta}}{r(s,n),n}\right)}_2^2 + 2\innerprod{\nabla_n \Lloss \left(\paramstep{\vect{\Theta}}{s}\right)}{\nabla_n \Lloss \left(\paramstep{\vecthat{\Theta}}{r(s,n),n}\right)} \\
		+ \mathrm{L}^2 	\norm{\paramstep{\vect{\Theta}}{s}- \paramstep{\vecthat{\Theta}}{r(s,n),n}}_2^2.
	\end{multline}
	We bound $2\innerprod{\nabla_n \Lloss \left(\paramstep{\vect{\Theta}}{s}\right)}{\nabla_n \Lloss \left(\paramstep{\vecthat{\Theta}}{r(s,n),n}\right)}$ using \eqref{eqn:ip_sum_ineq}. Note, $0 < \mu^2 < 1$. Thus,
	\begin{equation}
		\left(1-\mu^2\right)\norm{\nabla_n \Lloss \left(\paramstep{\vect{\Theta}}{s}\right)}_2^2 \leq \left(\frac{1}{\mu^2}-1\right)\norm{\nabla_n \Lloss \left(\paramstep{\vecthat{\Theta}}{r(s,n),n}\right)}_2^2 + \mathrm{L}^2 	\norm{\paramstep{\vect{\Theta}}{s}- \paramstep{\vecthat{\Theta}}{r(s,n),n}}_2^2.
		\label{eqn:theo-grad-avg-upperbound-noisy-v2-step1}
	\end{equation}
	%	Even though \eqref{eqn:ip_sum_ineq} is true for any $\mu>0$. For meaningful results we need $0<\mu<1$. This is because if we choose $\mu>1$ then we do not get a lower bound on $\norm{\nabla_n \Lloss \left(\paramstep{\vect{\Theta}}{s}\right)}_2^2$. Thus, we get,
%	Upon simplification we get,
%	\begin{equation}
%		\norm{\nabla_n \Lloss \left(\paramstep{\vect{\Theta}}{s}\right)}_2^2 \leq \frac{1}{\mu^2}\norm{\nabla_n \Lloss \left(\paramstep{\vecthat{\Theta}}{r(s,n),n}\right)}_2^2 + \frac{\mathrm{L}^2}{1-\mu^2}	\norm{\paramstep{\vect{\Theta}}{s}- \paramstep{\vecthat{\Theta}}{r(s,n),n}}_2^2.
%		\label{eqn:theo-grad-avg-upperbound-noisy-v2-step1}
%	\end{equation}
	Using \cref{corr-diff-delta-sq-twotau} on \eqref{eqn:theo-grad-avg-upperbound-noisy-v2-step1} we get
	\begin{equation}
		\norm{\nabla_n \Lloss \left(\paramstep{\vect{\Theta}}{s}\right)}_2^2 \\
		\leq \frac{1}{\mu^2}\norm{\nabla_n \Lloss \left(\paramstep{\vecthat{\Theta}}{r(s,n),n}\right)}_2^2 + \frac{\mathrm{L}^2(\tau + \frac{1}{2})}{1-\mu^2}	\sum_{i=\maxone{s-2\tau}}^{s-1} \norm{\paramstep{\vect{\Delta}}{i}}^2_2.
	\end{equation}
	Taking expectation $\expectation$ w.r.t. all. noise in steps $1,\dots,s-1$ and using \cref{lemma-lyapunov-diff-grad-square-noisy-special} we get,
	\begin{multline}
		\expectation \left[ \norm{\nabla_n \Lloss \left(\paramstep{\vect{\Theta}}{s}\right)}_2^2 \right]
		\leq \frac{1}{\mu^2}\left(\frac{\mathrm{L}}{\eta(1-\eta(\tau+1))}\expectation \left[ \xi_{r(s,n)} - \xi_{r(s,n)+1} \right] + \frac{\eta \left( \tau + 2 \right)}{4(1-\eta(\tau+1))} \sigma^2_n \right) \\
		+ \frac{\mathrm{L}^2(\tau + \frac{1}{2}) }{1-\mu^2}	\sum_{i=\maxone{s-2\tau}}^{s-1} \expectation \left[ \norm{\paramstep{\vect{\Delta}}{i}}^2_2 \right].
	\end{multline}
	Employing \cref{corr-lyapunov-diff-delta-square-noisy-special} we get,
	\begin{multline}
		\expectation \left[ \norm{\nabla_n \Lloss \left(\paramstep{\vect{\Theta}}{s}\right)}_2^2 \right] \leq\underbrace{ \frac{1}{\mu^2}\left(\frac{\mathrm{L}}{\eta(1-\eta(\tau+1))}\expectation \left[ \xi_{r(s,n)} - \xi_{r(s,n)+1} \right] + \frac{\eta \left( \tau + 2 \right)}{4(1-\eta(\tau+1))} \sigma^2_n \right)}_{a_1} \\
		+ \underbrace{\frac{2 \eta \mathrm{L} (2\tau + 1)}{(1-\mu^2)(4 - \eta (3\tau+4))}  \sum_{i=\maxone{s-2\tau}}^{s-1} \left( \expectation \left[\xi_s - \xi_{s+1} \right] + \frac{\eta}{\mathrm{L}}  \sigma^2_{m_i}\right)}_{a_2}	.
	\end{multline}
	Here, ${m_i} \in \{1,\dots,N+1\}$ denotes the block that was updated in the $i^{\text{th}}$ step. Let us denote, 
	\begin{equation}
		\frac{1}{S}\sum_{s=1}^{S} \sum_{l=1}^{N+1}\expectation \left[\norm{\nabla_n \Lloss 	\left(\paramstep{\vect{\Theta}}{s}\right)}_2^2\right] \leq A_1 + A_2
		\label{eqn:theo-grad-avg-upperbound-noisy-v2-grad-blockl-norm-bound}
	\end{equation}
	where $A_1 \defeq \frac{1}{S}\sum_{s=1}^{S} \sum_{l=1}^{N+1} a_1$ and $A_2 \defeq \frac{1}{S}\sum_{s=1}^{S} \sum_{l=1}^{N+1} a_2$.
%	is defined as
%	\begin{equation}
%		A_1 \defeq \frac{1}{S\mu^2} \sum_{l=1}^{N+1}\sum_{s=1}^{S} \left(\frac{\mathrm{L}}{\eta(1-\eta(\tau+1))}\expectation \left[ \xi_{r(s,n)} - \xi_{r(s,n)+1} \right] + \frac{\eta \left( \tau + 2 \right)}{4(1-\eta(\tau+1))} \sigma^2_n \right),
%	\end{equation}
%	$A_2$ is defined as
%	\begin{equation}
%		A_2 \defeq \frac{2 \eta \mathrm{L} (2\tau + 1)}{S(1-\mu^2)(4 - \eta (3\tau+4))} \sum_{l=1}^{N+1} \sum_{s=1}^{S} \sum_{i=\maxone{s-2\tau}}^{s-1} \left( \expectation \left[\xi_s - \xi_{s+1} \right] + \frac{\eta}{\mathrm{L}}  \sigma^2_{m_i}\right),
%		\label{eqn:theo-grad-avg-upperbound-noisy-v2-A2}
%	\end{equation}
	
	\textbf{Bounding $A_1$:} Bounding $A_1$ requires us to bound the sum
	\begin{equation}
		\sum_{s=1}^{S} \left(\frac{\mathrm{L}}{\eta(1-\eta(\tau+1))}\expectation \left[ \xi_{r(s,n)} - \xi_{r(s,n)+1} \right] + \frac{\eta \left( \tau + 2 \right)}{4(1-\eta(\tau+1))} \sigma^2_n \right).
	\end{equation}
	The number of times a particular term $ \left(\xi_{r(s,n)} - \xi_{r(s,n)+1}\right)$ is repeated in the summation depends on how many update steps the value of $r(s,n)$ can remain the same. In our setup we know that for every $N'E$ parameter update steps, $E$ of them have to be updates corresponding to block-$l$. So the maximum gap between changes in $r(s,n)$ is bounded by $2NE < \tau$. Thus we have
	\begin{multline}
		\sum_{s=1}^{S} \left(\frac{\mathrm{L}}{\eta(1-\eta(\tau+1))}\expectation \left[ \xi_{r(s,n)} - \xi_{r(s,n)+1} \right] + \frac{\eta \left( \tau + 2 \right)}{4(1-\eta(\tau+1))} \sigma^2_n \right) \\
		< \tau \sum_{s=1:r(s,n)\neq r(s-1,n)}^{S} \left(\frac{\mathrm{L}}{\eta(1-\eta(\tau+1))}\expectation \left[ \xi_{s} - \xi_{s+1} \right] + \frac{\eta \left( \tau + 2 \right)}{4(1-\eta(\tau+1))} \sigma^2_n \right).
		\label{eqn:theo-grad-avg-upperbound-noisy-v2-step3}
	\end{multline}
	Note that the term inside the summation is $\geq 0$ due to \cref{lemma-lyapunov-diff-grad-square-noisy-special}. Since only one block-$l$ is updated during a step $s$, $\sum_{l=1}^{N+1} \sum_{s=1:r(s,n)\neq r(s-1,n)}^{S} $ becomes $\sum_{s=1}^{S}$. Thusx $\sum_{l=1}^{N+1} \sum_{s=1:r(s,n)\neq r(s-1,n)}^{S} \sigma^2_n$ becomes  $\sum_{s=1}^{S} \sigma^2_{m_s}$ where $m_s$ is the block updated at the $s^{\text{th}}$ iteration.
	\begin{multline}
		\sum_{l=1}^{N+1} \sum_{s=1}^{S} \left(\frac{\mathrm{L}}{\eta(1-\eta(\tau+1))}\expectation \left[ \xi_{r(s,n)} - \xi_{r(s,n)+1} \right] + \frac{\eta \left( \tau + 2 \right)}{4(1-\eta(\tau+1))} \sigma^2_n \right) \\
		< \tau \sum_{s=1}^{S} \left(\frac{\mathrm{L}}{\eta(1-\eta(\tau+1))}\expectation \left[ \xi_{s} - \xi_{s+1} \right] + \frac{\eta \left( \tau + 2 \right)}{4(1-\eta(\tau+1))} \sigma^2_{m_s} \right).
		\label{eqn:theo-grad-avg-upperbound-noisy-v2-step3b}
	\end{multline}
	The first term is a telescopic sum and $\xi_s \geq 0$ by definition. Thus
	\begin{equation}
		A_1 < \frac{\tau}{S\mu^2} \left( \frac{\mathrm{L}}{\eta(1-\eta(\tau+1))}\xi_{1} + \frac{\eta \left( \tau + 2 \right)}{4(1-\eta(\tau+1))} \sum_{s=1}^S \sigma^2_{m_s}\right).
	\end{equation}
	Let us define $T \defeq \floor{\frac{S}{(N+1)E}}$. Then, $\frac{1}{S}\sum_{s=1}^S \sigma^2_{m_s} < \frac{E(T+1)}{S} \sigma^2 < \frac{2}{N+1} \sigma^2$. Thus we can write the bound on $A_1$ as
	\begin{equation}
		A_1 < \frac{\tau \mathrm{L}}{S \mu^2 \eta(1-\eta(\tau+1))}\xi_{1} + 	
		\frac{\eta \left( \tau + 2 \right)}{2 \mu^2 (N+1)(1-\eta(\tau+1))} \sigma^2.
		\label{eqn:theo-grad-avg-upperbound-noisy-v2-A1-bound}
	\end{equation}

	\textbf{Bounding $A_2$:}
	\begin{multline}
		\sum_{s=1}^{S} \sum_{i=\maxone{s-2\tau}}^{s-1}\expectation \left[\xi_s - \xi_{s+1} \right] + \frac{\eta}{\mathrm{L}}  \sigma^2_{m_i} =  \underbrace{\sum_{s=1}^{2\tau} \left(\xi_1 - \expectation \left[\xi_{s}\right] \right)}_{B_1} + \underbrace{\sum_{s=2\tau+1}^{S} \expectation \left[ \xi_{s-2\tau} -\xi_{s} \right] - \frac{\eta}{\mathrm{L}} \sum_{s=1}^{2\tau} \sum_{h=1}^{s-1} \sigma^2_{m_h}}_{B_2} \\
		+ \frac{\eta}{\mathrm{L}} \underbrace{\sum_{s=1}^{2\tau} \sum_{h=1}^{s-1} \sigma^2_{m_h}}_{B_3}+ \frac{\eta}{\mathrm{L}} \underbrace{\sum_{s=1}^{S} \sum_{i=\maxone{s-2\tau}}^{s-1}\sigma^2_{m_i} }_{B_4}
	\end{multline}
	Since, $\xi_s \geq 0 \forall s$
	\begin{equation}
		B_1 \leq (2\tau-1)\xi_1.
		\label{eqn:theo-grad-avg-upperbound-noisy-v2-B1-bound}
	\end{equation}
	When $S \leq 2\tau$, $B_2 = 0$. When $S \leq 2\tau$, $B_2=0$. Let us look at $B_2$ when $S \geq 2\tau+1$. We can write $B_2$ as 
	\begin{equation}
		B_2 = \sum_{j=1}^{S-2\tau} \expectation \left[\xi_{j} \right] - \sum_{s=2\tau+1}^{S} \expectation \left[\xi_{s} \right] -\sum_{s=1}^{2\tau} \sum_{h=1}^{s-1} \sigma^2_{m_h}.
		\label{eqn:theo-grad-avg-upperbound-noisy-v2-B2}
	\end{equation}
	where $j=s-2\tau$. There are further two cases here. 
	\begin{enumerate}
		\item If $S-2\tau \leq 2\tau$. In such a case the two summations in  \eqref{eqn:theo-grad-avg-upperbound-noisy-v2-B2} will not have any common terms. Thus, the second summation of negative terms can be dropped to get an upper bound. Thus, $\sum_{j=1}^{2\tau} \expectation \left[\xi_{j}\right] - \sum_{s=1}^{2\tau} \sum_{h=1}^{s-1} \sigma^2_{m_h}$ is an upper bound.
		\item  If $S-2\tau \geq 2\tau+1$. In this case, define $S-2\tau = 2\tau+1+m$ where $m \in \mathbb{N}$. After cancelling the common terms and dropping the remaining terms in the second summation, we get an upper bound on  \eqref{eqn:theo-grad-avg-upperbound-noisy-v2-B2} as $B_2 \leq \sum_{j=1}^{S-2\tau-m-1} \expectation \left[\xi_{j}\right] - \sum_{s=1}^{2\tau} \sum_{h=1}^{s-1} \sigma^2_{m_h}$. By definition of $m$, $S-2\tau-m-1=2\tau$.
	\end{enumerate}
	By repeatedly applying \Cref{corr-lyapunov-diff-delta-square-noisy-special}, we can show that $\xi_1 + \frac{\eta}{\mathrm{L}} \sum_{h=1}^{s-1} \sigma^2_{m_h} \geq \expectation \left[\xi_s\right]$. Thus,
	\begin{equation}
		B_2 \leq \sum_{j=1}^{2\tau} \expectation \left[\xi_{j} \right] -\frac{\eta}{\mathrm{L}} \sum_{h=1}^{s-1} \sigma^2_{m_j} \leq 2\tau \xi_1.
		\label{eqn:theo-grad-avg-upperbound-noisy-v2-B2-bound}
	\end{equation}
	To bound $B_3$ we recognize that
	\begin{multline}
		B_3 = (2\tau-1) \sigma^2_{m_1} + (2\tau-2) \sigma^2_{m_2} + \dots + \sigma^2_{m_{2\tau-1}}
		\leq (2\tau-1) \left( \sigma^2_{m_1} + \dots +\sigma^2_{m_{2\tau-1}} \right) \leq 8E\left(\tau-\frac{1}{2}\right)\sigma^2.
		\label{eqn:theo-grad-avg-upperbound-noisy-v2-B3-bound}
	\end{multline}
	The last inequality can be reasoned as follows. The period $2\tau-1$ corresponds to $4NE+2E-1$ update steps which is less than $4(N+1)E$. Hence the maximum number of times any block is updated is at most $4E$ updates. Making use of the definition of $\sigma^2 = \sum_{p=1}^{N+1}\sigma^2_p$ we get the bound.
	
	The inner summation of $B_4$ is a summations of $\sigma^2_{m_i}$ corresponding to a maximum length of $2\tau$. Following the same steps as the $B_3$ bound, this inner summation can also be bounded by $4E\sigma^2$. Thus,
	\begin{equation}
		B_4 \leq 4ES\sigma^2.
		\label{eqn:theo-grad-avg-upperbound-noisy-v2-B4-bound}
	\end{equation}
	
	Putting together \eqref{eqn:theo-grad-avg-upperbound-noisy-v2-B1-bound}, \eqref{eqn:theo-grad-avg-upperbound-noisy-v2-B2-bound}, \eqref{eqn:theo-grad-avg-upperbound-noisy-v2-B3-bound}, and \eqref{eqn:theo-grad-avg-upperbound-noisy-v2-B4-bound} we get
	\begin{equation}
		A_2 \leq \frac{2 \eta \mathrm{L} (N+1) (2\tau + 1)}{(1-\mu^2)(4 - \eta (3\tau+4))} \left(\frac{1}{S} \left((4\tau-1)\xi_1 + \frac{8\eta E\left(\tau -\frac{1}{2}\right)}{\mathrm{L}} \sigma^2 \right) + \frac{4 \eta E}{\mathrm{L}}\sigma^2\right)
		\label{eqn:theo-grad-avg-upperbound-noisy-v2-A2-bound}
	\end{equation}
	
	Putting together the bound on $A_1$ from \eqref{eqn:theo-grad-avg-upperbound-noisy-v2-A1-bound}
	and $A_2$ from \eqref{eqn:theo-grad-avg-upperbound-noisy-v2-A2-bound} we get
	\begin{multline}
		\frac{1}{S}\sum_{s=1}^{S} \expectation \left[\norm{\nabla \Lloss \left(\paramstep{\vect{\Theta}}{s}\right)}_2^2\right] < \frac{\tau \mathrm{L}}{S \mu^2 \eta(1-\eta(\tau+1))}\xi_{1} + 	
		\frac{\eta \left( \tau + 2 \right)}{2 \mu^2 (N+1)(1-\eta(\tau+1))} \sigma^2 \\
		+ \frac{2 \eta \mathrm{L} (N+1) (2\tau + 1)}{(1-\mu^2)(4 - \eta (3\tau+4))} \left(\frac{1}{S} \left((4\tau-1)\xi_1 + \frac{8\eta E\left(\tau -\frac{1}{2}\right)}{\mathrm{L}} \sigma^2 \right) + \frac{4 \eta E}{\mathrm{L}}\sigma^2\right).
		\label{eqn:theo-grad-avg-upperbound-noisy-v2-grad-blockl-sum-bound}
	\end{multline}
	Rearranging the terms in above we get
	\begin{multline}
		\frac{1}{S}\sum_{s=1}^{S} \expectation \left[\norm{\nabla \Lloss \left(\paramstep{\vect{\Theta}}{s}\right)}_2^2\right] < \frac{\xi_1}{S} \left( \frac{\tau \mathrm{L}}{\mu^2 \eta(1-\eta(\tau+1))} + \frac{2 \eta \mathrm{L} (N+1) (2\tau + 1)(4\tau - 1)}{(1-\mu^2)(4 - \eta (3\tau+4))} \right)
		\\
		+ \sigma^2 \left( \frac{\eta \left( \tau + 2 \right)}{2 \mu^2 (N+1)(1-\eta(\tau+1))} + \frac{2 \eta (N+1) (2\tau + 1)}{(1-\mu^2)(4 - \eta (3\tau+4))} \left( \frac{8\eta E\left(\tau -\frac{1}{2}\right)}{S} + 4 \eta E \right) \right).
	\end{multline}
\end{proof}

%\bibliographystyle{IEEEtran}
%\bibliography{paper}

\end{document}